\documentclass{article}

\usepackage{microtype}
\usepackage{graphicx}
\usepackage{subcaption}
\usepackage{booktabs} 

\usepackage{hyperref}

\usepackage[preprint]{icml2026}

\usepackage{amsmath}
\usepackage{amssymb}
\usepackage{mathtools}
\usepackage{amsthm}

\usepackage{color}

\usepackage{times}
\usepackage{bm}
\usepackage{natbib}
\usepackage{algorithm}
\usepackage{algorithmic}
\usepackage{mathrsfs}
\usepackage{graphicx}
\usepackage{subcaption}
\usepackage{bbm}
\usepackage[flushleft]{threeparttable}
\usepackage{comment}
\usepackage{arydshln}
\usepackage{pstricks}
\usepackage{tikz}
\usepackage[inline]{enumitem}

\usepackage{xcolor}         
\definecolor{hcolor}{cmyk}{0, 0.85, 0.85, 0.50}

\usepackage{transparent}
\definecolor{adbskyblue}{HTML}{c1d8f0}
\definecolor{adbskyyellow}{HTML}{F2D22E}
\definecolor{adbskyred}{HTML}{ecccad}

\newcommand{\xcolorbox}[2]{%
  \tikz[baseline=(T.base)]\node[fill=#1, fill opacity=0.3, text opacity=1, inner sep=2pt, anchor=base] (T) {\ensuremath{#2}};%
}

\usepackage[capitalize,noabbrev]{cleveref}

\newcommand{\new}{\textrm{new}}

\def\T{{ \mathrm{\scriptscriptstyle T} }}

\def\##1\#{\begin{align}#1\end{align}}
\def\$#1\${\begin{align*}#1\end{align*}}

\newcommand{\tr}{\operatorname{tr}}

\def\T{{\top}} 

\newcommand{\Rom}[1]{\text{\uppercase\expandafter{\romannumeral #1\relax}}}

\newcommand{\maxiter}{K}
\newcommand{\ndata}{n}
\newcommand{\ndim}{p}
\newcommand{\aspratio}{\rho}
\newcommand{\effreg}{\tau}
\newcommand{\effnoise}{\gamma}

\newcommand{\siglev}{r}
\newcommand{\niter}{t}
\newcommand{\effprojmat}{Q}
\newcommand{\deteffnoise}{D}

\newcommand {\covmat}{\Sigma}

\newcommand {\sigval}{s}

\newcommand{\noiselev}{\sigma^2}

\newcommand{\cD}{\mathcal{D}}
\newcommand{\RR}{\mathbb{R}}
\newcommand{\EE}{\mathbbm{E}}
\newcommand{\argmin}{\mathop{\mathrm{argmin}}}

\renewcommand{\hat}{\widehat}

\theoremstyle{plain}
\newtheorem{theorem}{Theorem}[section]

\newtheorem{lemma}[theorem]{Lemma}
\newtheorem{corollary}[theorem]{Corollary}
\theoremstyle{definition}
\newtheorem{definition}[theorem]{Definition}
\newtheorem{assumption}[theorem]{Assumption}
\theoremstyle{remark}

\usepackage[textsize=tiny]{todonotes}

\icmltitlerunning{Why Self-Training Helps and Hurts}

\begin{document}

\twocolumn[
 
  \icmltitle{Why Self-Training Helps and Hurts: Denoising vs. Signal Forgetting}



  \icmlsetsymbol{equal}{*}

  \begin{icmlauthorlist}
    \icmlauthor{Mingqi Wu}{yyy}
    \icmlauthor{Archer Y. Yang}{yyy}
    \icmlauthor{Qiang Sun}{sch}
  \end{icmlauthorlist}

  \icmlaffiliation{yyy}{Mcgill University and Mila}
  \icmlaffiliation{sch}{University of Toronto and MBZUAI}

\icmlcorrespondingauthor{Mingqi Wu}{minkiw146@gmail.com}
  \icmlcorrespondingauthor{Archer Y. Yang}{archer.yang.yi@gmail.com}
  \icmlcorrespondingauthor{Qiang Sun}{qsunstats@gmail.com}

  \icmlkeywords{Machine Learning, ICML}

  \vskip 0.3in
]



\printAffiliationsAndNotice{}  

\begin{abstract}


  Iterative self-training (self-distillation) repeatedly refits a model on pseudo-labels generated by its own predictions. We study this procedure in overparameterized linear regression: an initial estimator is trained on noisy labels, and each subsequent iterate is trained on fresh covariates with noiseless pseudo-labels from the previous model. In the high-dimensional regime, we derive deterministic-equivalent recursions for the prediction risk and effective noise across iterations, and prove that the empirical quantities concentrate sharply around these limits. The recursion separates two competing forces: a systematic component that grows with iteration due to progressive \emph{signal forgetting}, and a stochastic component that decays due to \emph{denoising} via repeated data-dependent projections. Their interaction yields a $U$-shaped test-risk curve and an optimal early-stopping time. In spiked covariance models, iteration further acts as an iteration-dependent spectral filter that preserves strong eigendirections while suppressing weaker ones, inducing an implicit form of soft feature selection distinct from ridge regression. Finally, we propose an \emph{iterated} generalized cross-validation criterion and prove its uniform consistency for estimating the risk along the self-training trajectory, enabling fully data-driven selection of the stopping time and regularization. Experiments on synthetic covariances validate the theory and illustrate the predicted denoising–forgetting trade-off.

\end{abstract}

\section{Introduction}
\label{sec:introduction}

Iterative self-training, repeatedly refining a model using its own predictions, is a powerful yet paradoxical paradigm in modern machine learning. On the one hand, it has a long record of success in semi-supervised learning, where it leverages large unlabeled datasets to achieve state-of-the-art performance~\citep{scudder1965probability, yarowsky1995unsupervised, xie2020self, sohn2020fixmatch}. On the other hand, the same recursion can fail catastrophically. Recent work identifies a major failure mode, \emph{model collapse}, in which training on synthetic, self-generated outputs progressively degrades performance and can lead the model to forget the original data distribution~\citep{shumailov2023curse, dohmatob2024model}. Mitigating collapse, for example by mixing real data into each iteration, is an active research direction~\citep{garg2025preventing, gerstgrasser2024is}. These observations expose a fundamental tension that motivates our central question:
\begin{center}
\textit{{\color{hcolor}When does iteration improve generalization, \\and when does it hurt?}}
\end{center}
A useful intuition is to view each self-training step as a lossy teacher--student transfer. The student is trained only on the teacher's predictions, so each iteration necessarily discards information: some discarded components correspond to noise, but others contain signal. In high dimensions, this transfer resembles repeatedly compressing (or projecting) the current estimate through a fresh, data-dependent subspace. Unstable, noise-like components are removed quickly, whereas strong signal directions that are consistently supported by the data geometry persist for several iterations. When the signal is strong, early iterations remove more noise than signal and can improve generalization; once most noise has been removed, further iterations begin to erode signal as well, increasing systematic error. This trade-off yields the characteristic $U$-shaped risk curve and motivates early stopping.

These dynamics are closely connected to several core concepts in learning theory. Iterative self-training can be viewed as a form of knowledge distillation~\citep{Hinton2015Distilling}, where the estimator at step $t$ serves as the teacher for step $t{+}1$. It also instantiates weak-to-strong  generalization~\citep{burns2024weak}, in which a stronger model is trained using labels produced by a weaker supervisor. Despite strong empirical performance, a precise statistical characterization of the iterative mechanism remains open.

\begin{figure}[t]
    \centering
    \begin{subfigure}[b]{0.485\linewidth}
        \centering
        \includegraphics[width=\linewidth]{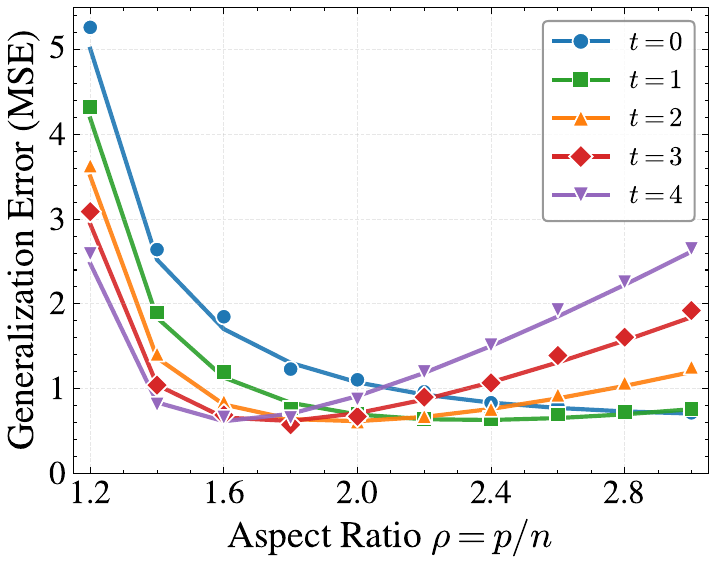}
        \caption{Error vs.\ aspect ratio}
        \label{fig:sim_aspect_ratio}
    \end{subfigure}
    \hfill
    \begin{subfigure}[b]{0.48\linewidth}
        \centering
        \includegraphics[width=\linewidth]{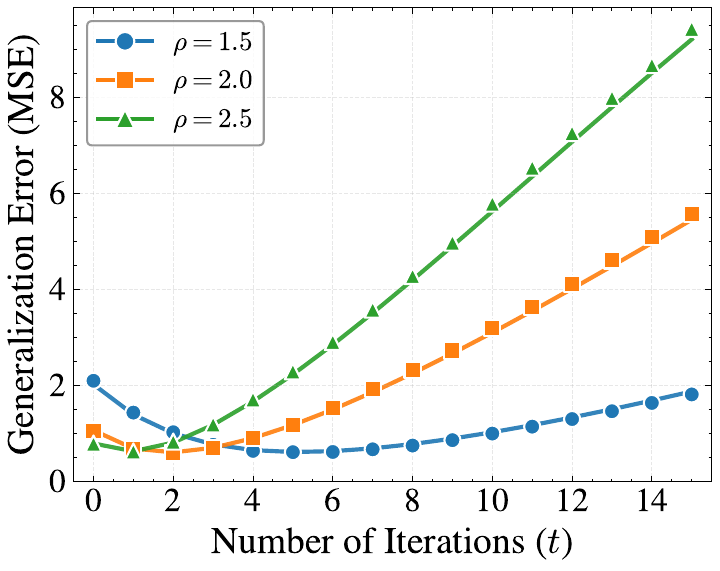}
        \caption{Error vs.\ iterations}
        \label{fig:sim_iterations}
    \end{subfigure}
    \caption{\small \textbf{Simulation results for the spiked covariance model ($s=25$).}
    Solid lines denote theoretical predictions; markers show simulation averages over 10 trials.
    \textbf{(a)} Iteration substantially reduces generalization error in the overparameterized regime.
    \textbf{(b)} The $U$-shaped curves match our theoretical predictions: initial iterations reduce stochastic error, while excessive iterations eventually lead to an accumulation of systematic error.}
    \label{fig:generalization_error_science_style}
\end{figure}

In this paper, we answer this question in an analytically tractable setting, overparameterized linear regression, using tools from high-dimensional asymptotic theory~\citep{hastie2022surprises, han2023distribution, chen2023sketched, wu2025ensemble}. We focus on the proportional asymptotic regime where $\ndim > \ndata$. Figure~\ref{fig:generalization_error_science_style} previews the central trade-off: early iterations denoise by attenuating stochastic error inherited from the initial noisy fit, but subsequent updates rely only on pseudo-labels and therefore gradually discard signal (\emph{signal forgetting}), which accumulates as systematic error. The resulting competition between denoising and signal forgetting produces a characteristic $U$-shaped test-error curve as a function of the number of iterations and implies an optimal stopping time.

Our main contributions are fourfold. (i) We derive a rigorous decomposition of the generalization error for iterative self-training in overparameterized linear regression, showing that the error splits into two parts: a systematic error that increases with iteration (\emph{signal forgetting}) and a stochastic error that decreases (\emph{denoising}). (ii) We prove that the competition between these two parts yields a characteristic $U$-shaped risk curve and implies an optimal number of iterations, explaining why excessive self-training is harmful without proposing a specific tuning procedure. (iii) We identify the mechanism driving these dynamics as an implicit regularization effect distinct from ridge regression: self-training preserves strong signal directions aligned with the feature covariance while suppressing stochastic components that lie in large, randomly oriented null spaces of the design matrices. This perspective explains both the denoising behavior and the student–teacher performance gap. (iv) We propose an iterated generalized cross-validation (GCV) procedure to estimate the optimal stopping time.

\paragraph{Related Work}

We briefly review the works most closely related to ours.

\textbf{Self-training.}
Self-training is a classical semi-supervised paradigm that augments limited labeled data with \emph{pseudo-labels} generated on unlabeled examples~\citep{scudder1965probability, yarowsky1995unsupervised}. It remains a core ingredient in modern deep-learning pipelines, including \emph{Noisy Student} and \emph{FixMatch}~\citep{xie2020self, sohn2020fixmatch}, with later methods improving pseudo-label reliability via adaptive selection or weighting~\citep{wang2023freematch, chen2023softmatch}. Extensions cover structured prediction tasks such as segmentation~\citep{chen2021semi}.

\textbf{Theory and collapse under recursion.}
Recent theory characterizes when self-training can amplify weak predictors under distributional assumptions~\citep{frei2022self, wei2021theoretical, zhang2022how}. A contrasting line studies \emph{model collapse} under recursive training on self-generated data~\citep{shumailov2023curse, bertrand2024on, dohmatob2024model}. Mitigation via mixing synthetic and fresh real data has been analyzed in overparameterized linear regression, including optimal mixing ratios and conditions for bounded error~\citep{garg2025preventing, he2025golden, gerstgrasser2024is, dey2024universality}.


\section{Preliminaries}\label{sec:pre}


This section introduces the setup for our analysis. We first define the high-dimensional linear model and the prediction risk. We then describe the \emph{ridge} and \emph{ridgeless} estimators, which underpin the iterative self-training procedure studied in this paper. Finally, we state our assumptions on the data-generating process.


\subsection{The linear model and prediction risk}
Consider a high-dimensional linear model with training data
$\cD_{0} := \{(x_i, y_i)\}_{i=1}^{\ndata}$. The observations are i.i.d. samples generated by: 
\begin{align}\label{eq:lm}
&y_i = x_i^\T \beta + \varepsilon_i, ~~  i = 1,\ldots, n, \\
&\nonumber \text{where }(x_i, \varepsilon_i) \sim P_x\times P_\varepsilon.
\end{align} 
Here $P_x$ is a distribution on $\RR^\ndim$, and  $P_\varepsilon$ is a distribution on $\RR$ with mean $0$ and variance $\noiselev$.  In matrix notation, we write
\begin{equation*}
    Y_{0} = X_{0} \beta + E_{0},
\end{equation*}
with  $Y_{0}=(y_1,\ldots, y_n)^\T \in \RR^{\ndata}$, $X_{0}=(x_1,\ldots, x_n)^\T \in \RR^{\ndata \times \ndim}$, and $E_{0} = (\varepsilon_1,\ldots, \varepsilon_n)^\T \in \RR^{\ndata}$.

To evaluate an estimator $\hat{\beta}$, we use the (out-of-sample) prediction risk, defined as the expected squared error on an independent test point $x_{\new}\sim P_x$:
\begin{equation}\label{eq:riskcon}
\!\!\!\! \mathcal{R}(\hat{\beta}; \beta) 
\!=\! \EE\!\left[\left(x_\new^\top \hat{\beta} - x_\new ^\top \beta \right)^2  \right]
\!=\! \left\| \covmat^{1/2}(\hat{\beta} - \beta) \right\|_{2}^{2},
\end{equation}
where $\covmat = \EE[x_\new x_\new^\top]$
 is the population feature covariance. The expectation is over $x_{\new}$ conditional on the training data.

\subsection{The iterative self-training algorithm}
This section introduces the iterative self-training procedure and the associated estimators. We begin with the ridge and ridgeless estimators~\citep{hastie2022surprises}.

\paragraph{Standard ridge and ridgeless estimators.}
We focus on two standard estimators. The ridge estimator is the solution to
\$
\hat{\beta}_{0}^{\lambda}
&:= \argmin_{b\in\RR^{\ndim}}
\left\{
\frac{1}{2\ndata}\|Y_{0}-X_{0}b\|_{2}^{2}
+\frac{\lambda}{2}\|b\|_{2}^{2}
\right\} \\
&= (X_{0}^{\top}X_{0}+\ndata\lambda I)^{-1}X_{0}^{\top}Y_{0},
\$
where $\lambda>0$ and $I\in\RR^{\ndim\times\ndim}$ is the identity matrix.

In the limit $\lambda\to 0^{+}$, ridge converges to the ridgeless (minimum-norm) least-squares estimator~\citep{hastie2022surprises},
\begin{align*}
\hat{\beta}_{0}
&:= \argmin_{b\in\RR^{\ndim}} \Big\{ \|b\|_{2}:
b \text{  minimizes } \|Y_{0}-X_{0}b\|_{2}^{2} \Big\} \\
&= (X_{0}^{\top}X_{0})^{+}X_{0}^{\top}Y_{0}
= \lim_{\lambda\to 0^{+}} (X_{0}^{\top}X_{0}+\ndata\lambda I)^{-1}X_{0}^{\top}Y_{0},
\end{align*}
where $(\cdot)^{+}$ denotes the Moore--Penrose pseudoinverse. 

\paragraph{Iterative self-training.}
We now introduce the main object of study: an iterative procedure that refines an initial linear estimator. Starting from $\hat{\beta}_{0}$ fit on noisy data, each iteration draws a fresh feature matrix $X_{\niter}$, produces pseudo-labels $Y_{\niter}=X_{\niter}\hat{\beta}_{\niter-1}$, and fits a new estimator on the resulting synthetic (noiseless) dataset. The procedure is summarized in Algorithm~\ref{alg:self_train}.

\begin{algorithm}[tb]
\caption{Iterative self-training.}
\label{alg:self_train}
\begin{algorithmic}
\STATE \textbf{Input:} $\cD_{0}:=\{X_{0},Y_{0}\}$, regularization parameters $\{\lambda_{\niter}\}_{\niter=0}^{\maxiter}$, iterations $\maxiter$.
\STATE \textbf{Initialization:} $\hat{\beta}_{0}\leftarrow (X_{0}^{\top}X_{0}+\ndata\lambda_{0} I)^{+}X_{0}^{\top}Y_{0}$.
\FOR{$\niter=1,2,\dots,\maxiter$}
\STATE Sample $X_{\niter}\in\RR^{\ndata_{\niter}\times\ndim}$ independently.
\STATE Set $Y_{\niter}\leftarrow X_{\niter}\hat{\beta}_{\niter-1}$.
\STATE Update $\hat{\beta}_{\niter}\leftarrow (X_{\niter}^{\top}X_{\niter}+\ndata_{\niter}\lambda_{\niter} I)^{+}X_{\niter}^{\top}Y_{\niter}$.
\ENDFOR
\STATE \textbf{Output:} $\{\hat{\beta}_{0},\hat{\beta}_{1},\dots,\hat{\beta}_{\maxiter}\}$.
\end{algorithmic}
\end{algorithm}

\paragraph{Self-training as self-distillation.}
Algorithm~\ref{alg:self_train} can be viewed as self-distillation: $\hat{\beta}_{\niter-1}$ acts as a teacher that generates pseudo-labels for fresh features, and $\hat{\beta}_{\niter}$ is the student trained on these labels. Because $Y_{\niter}$ contains no additive noise, each step can attenuate the stochastic error inherited from the initial fit, but it may also discard signal through repeated teacher--student transfer. Our analysis tracks the evolution of the prediction risk of $\{\hat{\beta}_{\niter}\}_{\niter=0}^{\maxiter}$ across iterations.

\subsection{Assumptions}
Our analysis relies on the following assumptions on the data-generating process.

\begin{assumption}[Independent feature matrices]\label{ass:indpX}
The feature matrices $\{X_{\niter}\}_{\niter\ge 0}$ used across iterations are mutually independent.
\end{assumption}

\begin{assumption}[Gaussian covariates]\label{ass:Gauss}
For each iteration $\niter\ge 0$, the feature matrix $X_{\niter}\in\RR^{\ndata_{\niter}\times\ndim}$ has i.i.d.\ rows $\{x_{\niter,i}\}_{i=1}^{\ndata_{\niter}}$ with
\[
x_{\niter,i}\sim \mathcal{N}(0,\covmat_{\niter}),
\]
where $\covmat_{\niter}\in\RR^{\ndim\times\ndim}$ is positive definite.
\end{assumption}

\begin{assumption}[Uniformly well-conditioned covariances]\label{ass:cov}
The training covariances $\{\covmat_{\niter}\}_{\niter\ge 0}$ and the test covariance $\covmat$ are uniformly well-conditioned: there exist constants $0<c_{\lambda}\le C_{\lambda}<\infty$ such that for every
$S\in \{\covmat_{\niter}\}_{\niter\ge 0}\cup\{\covmat\}$,
\[
c_{\lambda}\le \lambda_{\min}(S)\le \lambda_{\max}(S)\le C_{\lambda}.
\]
\end{assumption}

\begin{assumption}[Initial Gaussian noise]\label{ass:noise}
Noise is present only at $\niter=0$. Specifically, $E_{0}=(\varepsilon_1,\dots,\varepsilon_{\ndata})^\top$ satisfies
\[
E_{0}\sim \mathcal{N}(0,\noiselev I_{\ndata}),
\]
for some $\noiselev>0$, and is independent of $\{X_{\niter}\}_{\niter\ge 0}$.
\end{assumption}

Assumption~\ref{ass:indpX} isolates the teacher–student transfer effect: each iteration trains a new student on pseudo-labels using an independent feature matrix. This independence makes the dynamics analytically tractable and clarifies how information is discarded across iterations. Assumption~\ref{ass:Gauss} (Gaussian covariates) is primarily imposed for simplicity and can often be relaxed to bounded-moment conditions~\citep{hastie2022surprises, chen2023sketched}. Assumption~\ref{ass:cov} is a standard regularity condition that precludes degeneracy and ensures that the deterministic equivalents used in our high-dimensional analysis are well behaved. Finally, Assumption~\ref{ass:noise} (noise only at $\niter=0$) cleanly separates two forces: early iterations attenuate the stochastic error inherited from the initial noisy fit (denoising), while repeated self-training gradually erodes the signal component (signal forgetting).


\section{Warm-up: Risk Dynamics in the Spiked Covariance Model}\label{sec:toy}

To build intuition for the general theory in Section~\ref{sec:mainthm}, we first analyze the iterative self-training dynamics under the \emph{spiked covariance model}. This setting yields an interpretable description of how iteration denoises the estimate while gradually forgetting signal.

\textbf{Analysis framework.}
Unless otherwise stated, all results in this section share the following settings:
\begin{itemize}
    \item \emph{Ridgeless limit:} We study the minimum-norm interpolator by taking $\lambda_{\niter}\to 0^+$ for all $\niter\ge 0$.
    \item \emph{Asymptotic overparameterization:} We consider the high-dimensional limit $\ndim,\ndata\to\infty$ with
    $\aspratio := \ndim/\ndata > 1$ fixed, and define the effective regularization $\effreg := \aspratio-1>0$.
\end{itemize}

We specialize the general framework from Section~\ref{sec:pre} to the following concrete model.

\begin{assumption}[Single-spiked covariance model]\label{ass:spikedmodel}
The data-generating process satisfies:
\begin{itemize}
    \item \emph{Spiked covariance:} For all $\niter\ge 0$, the feature covariance is
    $\Sigma = (\sigval-1)u_1u_1^\top + I_p$, where $\sigval>1$ is the spike strength and $u_1$ is the leading eigenvector.
    \item \emph{Aligned signal:} The signal is aligned with the spike: $\beta=\siglev u_1$, where $\siglev^2=\|\beta\|_2^2$ is the signal power.
\end{itemize}
\end{assumption}

\subsection{Asymptotic risk dynamics}

We now specialize the general theory to the spiked covariance model. Theorem~\ref{thm:Gthm} establishes that the empirical prediction risk concentrates around a deterministic limit $\mathcal R^{*}_{t}$. The following result characterizes this limit, revealing its explicit decomposition into competing systematic and stochastic components:
\begin{theorem}[Asymptotic risk recursion for the spiked covariance model]\label{thm:spiked_risk}
Assume Assumptions~\ref{ass:indpX}--\ref{ass:noise} and Assumption~\ref{ass:spikedmodel}. In the limit $\ndim,\ndata\to\infty$ with $\ndim/\ndata\to \aspratio>1$ and $\effreg = \aspratio-1$, the deterministic prediction risk
$\mathcal{R}_{\niter}^{*} := \lim_{\ndim,\ndata\to\infty}\mathcal{R}_{\niter}$ admits the decomposition
\begin{equation}\label{eq:spikerisk}
\mathcal{R}_{\niter}^{*}=\xcolorbox{adbskyyellow}{\mathcal{B}_{\niter}^{*}}+\xcolorbox{adbskyred}{\mathcal{V}_{\niter}^{*}},
\end{equation}
where the systematic error   is
\begin{equation}\label{eq:signalerror}
\xcolorbox{adbskyyellow}{\mathcal{B}_{\niter}^{*}}
=\siglev^{2}\sigval\left(1-\left(\frac{\sigval}{\sigval+\effreg}\right)^{\niter+1}\right)^{2},
\end{equation}
and for $\niter\ge 1$, the stochastic error   satisfies 
\begin{equation}\label{eq:stocherror}
\xcolorbox{adbskyred}{\mathcal{V}_{\niter}^{*}}
=\frac{\mathcal{V}_{\niter-1}^{*}}{1+\effreg}
+\effreg\,\siglev^{2}\frac{\sigval^{2\niter+1}}{(\sigval+\effreg)^{2(\niter+1)}}, 
\end{equation}
with initial condition
$\mathcal{V}_{0}^{*}
=\frac{\noiselev}{\effreg}
+\frac{\effreg\,\sigval}{(\sigval+\effreg)^{2}}\siglev^{2}$.
\end{theorem}

\begin{figure}[t]
\centering
\begin{subfigure}[b]{0.48\linewidth}
\centering
\includegraphics[width=\linewidth]{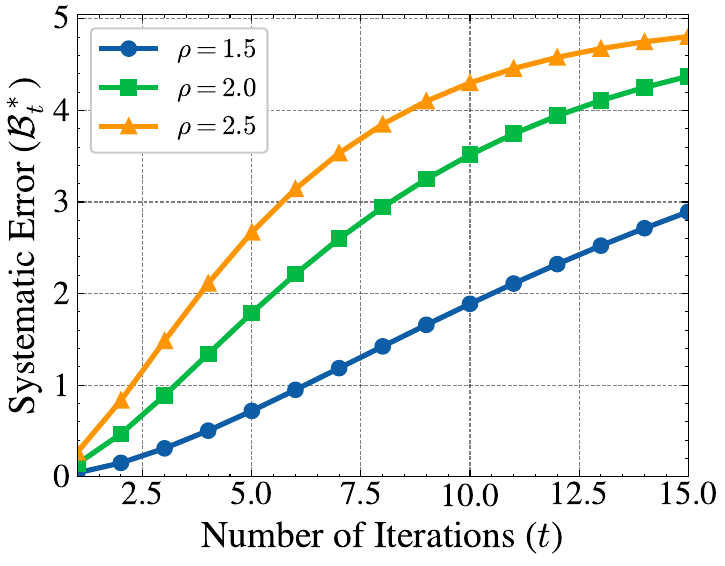}
\caption{Systematic error $\mathcal{B}_{\niter}^{*}$}
\label{fig:signal_error}
\end{subfigure}
\hfill
\begin{subfigure}[b]{0.48\linewidth}
\centering
\includegraphics[width=\linewidth]{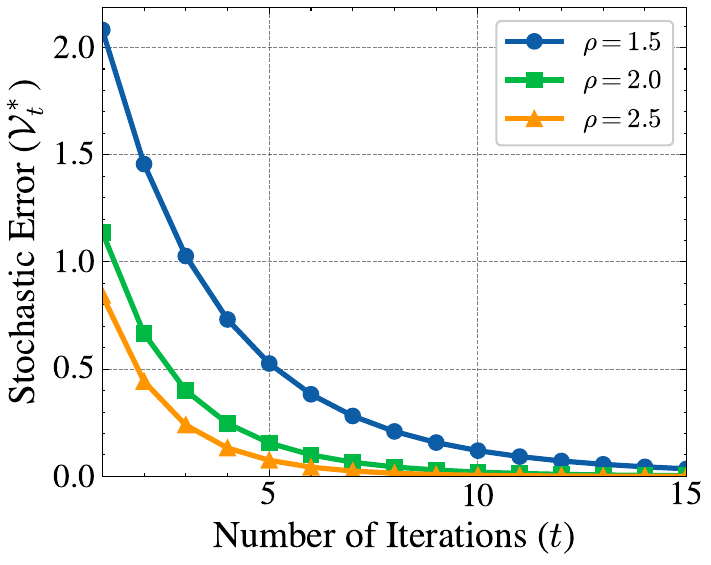}
\caption{Stochastic error $\mathcal{V}_{\niter}^{*}$}
\label{fig:stochastic_error}
\end{subfigure}
\caption{\small \textbf{Systematic vs.\ stochastic errors across iterations.}
The systematic error. \textbf{(a)} increases with $\niter$, reflecting signal forgetting, while the stochastic error. \textbf{(b)} decreases due to denoising. Each curve corresponds to a different aspect ratio $\aspratio$.}
\label{fig:error_vs_iterations}
\end{figure}

Theorem~\ref{thm:spiked_risk} decomposes the risk into two competing forces (as shown in Figure \ref{fig:error_vs_iterations}).
\textbf{Systematic error} $\xcolorbox{adbskyyellow}{\mathcal{B}_{\niter}^{*}}$ (signal error)  quantifies \emph{signal forgetting}: it tracks the deterministic decay of the estimator's projection onto the true signal direction $u_{1}$. This term increases with $t$.  As $\niter\to\infty$, it converges to
$\siglev^{2}\sigval$ -- the risk of the null estimator $\hat{\beta}=0$ -- implying that excessive iteration eventually erases the signal learned from the initial data.  \textbf{Stochastic error} $\xcolorbox{adbskyred}{\mathcal{V}_{\niter}^{*}}$ (noise error) captures the error arising from (i) the initial noise and (ii) the resampling of  feature matrices $X_t$ at each iteration.   
Crucially, even though the pseudo-labels $Y_{\niter}=X_{\niter}\hat{\beta}_{\niter-1}$ are noiseless conditional on $\hat{\beta}_{\niter-1}$, the student model injects fresh estimation error at every step by projecting the teacher's predictions onto a random, rank-deficient subspace.



The following corollary formalizes that iteration \emph{denoises} the stochastic component:
\begin{corollary}[Stochastic error reduction]\label{lm:stochreduct}
Under the assumptions of Theorem~\ref{thm:spiked_risk}, the stochastic error vanishes asymptotically: $\mathcal{V}_{\niter}^{*}\to 0$ as $\niter\to\infty$.
\end{corollary}
The decomposition in Theorem~\ref{thm:spiked_risk} thus highlights the fundamental trade-off:
each iteration reduces stochastic error $\mathcal{V}_{\niter}^{*}$ at the cost of increasing systematic error $\mathcal{B}_{\niter}^{*}$.

\textbf{Rate mismatch and the origin of the $U$-shape risk. }
In the signal-free limit $(r^2=0)$, systematic error $\mathcal{B}_{\niter}^{*}$ vanishes, and   stochastic error exhibits pure exponential decay: $\mathcal{V}_{\niter}^{*}=\frac{\sigma^2}{\tau(1+\tau)^t}$. Similarly, in the isotropic limit ($\sigval=1$, see Appendix \ref{sec:iso_cov}),  $\mathcal{B}_{\niter}^{*}$ and  $\mathcal{V}_{\niter}^{*}$ contract at identical rates, yielding monotonic risk $\mathcal{R}_{\niter}^{*}$ in $t$, as seen in Figure \ref{fig:comparison_results}a. The $U$-shape risk $\mathcal{R}_{\niter}^{*}$ therefore emerges strictly from \emph{anisotropy} ($s>1$), which induces a rate mismatch: noise error $\mathcal{V}_{\niter}^{*}$ decays rapidly at rate $(1+\effreg)^{-\niter}$, while the signal-forgetting term $\mathcal{B}_{\niter}^{*}$ grows much more slowly because contraction factor $\sigval/(\sigval+\effreg)$ is very close to one. As shown in Figure \ref{fig:comparison_results}a,  {\color{hcolor}\emph{this disparity creates a temporary window where denoising dominates before signal forgetting takes over.}}

\subsection{Implicit feature selection}\label{sec:implicitreg}

We now study the regime in which the signal is easily detectable, i.e., the spike strength $\sigval$ is large. In this regime, the iterates behave like a \emph{direction-adaptive spectral filter}: they preserve the dominant signal direction while repeatedly suppressing isotropic components induced by noise and resampling. The following corollary characterizes the risk as the spike strength diverges.

\begin{corollary}[Asymptotics for strong spikes]\label{lm:strongsp}
As $\sigval\to\infty$, the risk components satisfy
\begin{align*}
\mathcal{B}_{\niter}^{*} &\to 0, \quad \mathcal{V}_{\niter}^{*} \to \frac{\noiselev}{\effreg(1+\effreg)^{\niter}}.
\end{align*}
\end{corollary}
In this limit, the systematic error vanishes entirely. The algorithm achieves ideal separation: it perfectly preserves the signal component while attenuating the stochastic error at an exponential rate.

\begin{figure}[t]
\centering
\begin{subfigure}[b]{0.48\linewidth}
\centering
\includegraphics[width=\linewidth]{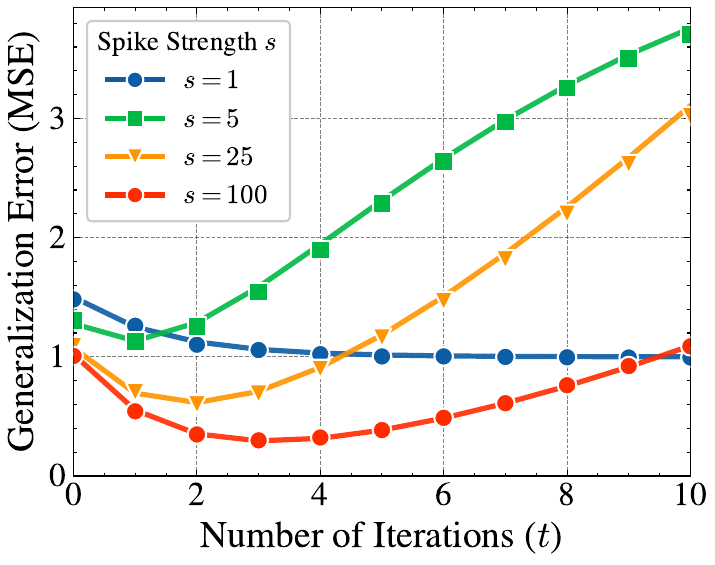}
\caption{Risk evolution}
\label{fig:strongspike}
\end{subfigure}
\hfill
\begin{subfigure}[b]{0.48\linewidth}
\centering
\includegraphics[width=\linewidth]{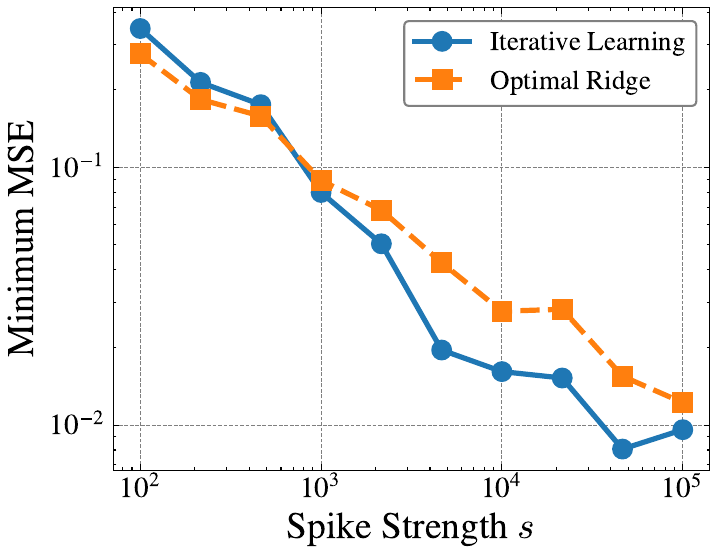}
\caption{Performance crossover}
\label{fig:optimalridge}
\end{subfigure}
\caption{\small \textbf{Performance under the spiked covariance model.}
\textbf{(a)} Test error over iterations $\niter$ for varying spike strengths $\sigval$. The $U$-shaped curves reflect early denoising and late-stage signal forgetting.
\textbf{(b)} Minimum error over $\niter$ versus optimally tuned ridge (minimum over $\lambda$). As signal strength $\sigval$ increases, iterative self-training can surpass ridge.}
\label{fig:comparison_results}
\end{figure}

\textbf{A spectral-filter viewpoint.}
Theorem~\ref{thm:spiked_risk} exposes the mechanism. The systematic component $\mathcal{B}_{\niter}^{*}$ in \eqref{eq:signalerror} is governed by the survival factor
$\left(\frac{\sigval}{\sigval+\effreg}\right)^{\niter+1}$.
When $\sigval\gg \effreg$, this factor is close to $1$, so the algorithm barely shrinks the signal, consistent with $\mathcal{B}_{\niter}^{*}$ being small. More precisely,  a first-order expansion yields
\[
1-\left(\frac{\sigval}{\sigval+\effreg}\right)^{\niter+1}
=1-\left(1-\frac{\effreg}{\sigval+\effreg}\right)^{\niter+1}
\approx \frac{(\niter+1)\effreg}{\sigval},
\]
and hence $\mathcal{B}_{\niter}^{*}\approx \siglev^{2}(\niter+1)^{2}\effreg^{2}/\sigval \to 0$ as $\sigval\to\infty$. This quantifies the crucial insight: {\color{hcolor}\emph{strong spikes prevent signal forgetting.}}

\textbf{Mechanism: repeated random projections.} To see why the stochastic component is reduced more aggressively, rewrite the ridgeless update as
\[
\hat{\beta}_{\niter}
=(X_{\niter}^{\top}X_{\niter})^{+}X_{\niter}^{\top}Y_{\niter}
=(X_{\niter}^{\top}X_{\niter})^{+}X_{\niter}^{\top}X_{\niter}\hat{\beta}_{\niter-1}
=:P_{\niter}\hat{\beta}_{\niter-1}.
\]
The operator $P_{\niter}$ is the orthogonal projection onto the row space of $X_{\niter}$ (equivalently, the span of the top $\ndata_{\niter}$ right singular vectors of $X_t$). In the underparameterized regime $\ndim\le \ndata_{\niter}$, $X_{\niter}^{\top}X_{\niter}$ is invertible and $P_{\niter}=I$, so the iterates remain unchanged; hence {\color{hcolor}\emph{nontrivial dynamics arise only when $\ndim>\ndata_{\niter}$}.} This projection acts differently on signal and noise components of the estimator:
\begin{enumerate}
\item \emph{Signal preservation.}
For large $\sigval$, the leading singular direction of $X_{\niter}$ aligns with $u_{1}$. Consequently, $P_{\niter}u_{1} \approx u_{1}$. The signal component lies largely in the projected subspace and is retained across iterations. 
\item \emph{Noise suppression.}
Noise components are approximately isotropic.
At each step, $P_{\niter}$ annihilates the noise component falling into its $(\ndim - \ndata_{\niter})$-dimensional null space. Because the feature matrices $\{X_{\niter}\}$ are independent, their null spaces are incoherent; repeated projection progressively shears away noise in all directions. 
\end{enumerate}

Together, these effects yield an \emph{implicit feature-selection} behavior: the algorithm preserves persistent, high-variance directions while progressively eliminating transient, randomly oriented components. This explains the algorithm's effectiveness in the strong-signal regime and parallels the mechanism of weak-to-strong generalization~\citep{burns2024weak}: {\color{hcolor}\emph{a student trained on labels produced by a teacher can denoise while preserving the teacher's reliable signal directions.}}

\textbf{Connection to ridge and the performance crossover.}
Ridge regression imposes \emph{global} shrinkage controlled by a single $\lambda$, shrinking all directions as a function of their empirical eigenvalues. In contrast, iterative self-training implements an \emph{iteration-dependent, direction-adaptive} filter: in the strong-spike regime it leaves the dominant signal direction nearly unshrunk while repeatedly contracting the isotropic subspace at an effective rate close to $(1+\effreg)^{-1}$ per iteration. This mechanistic difference explains the crossover in Figure~\ref{fig:optimalridge}: as $\sigval$ grows, self-training can denoise without incurring the signal shrinkage that globally regularized methods may still pay.

\subsection{Generalization to multi-spiked covariance}\label{subsec:multi_spike}

The single-spike model reveals the fundamental mechanism of signal preservation versus forgetting. However, real-world data typically possesses a richer spectral structure with multiple latent features of varying strengths. To capture this, we extend our analysis to the \emph{multi-spike covariance model}. This setting demonstrates that the iterative self-training procedure acts as a \textit{spectral filter}, selectively preserving strong signal components while suppressing weaker ones. 

\begin{assumption}[Multi-spiked covariance model]\label{ass:multi-spike}
Assume that, for all $t\geq 0$, the feature covariance and signal take the form
\[
\Sigma_t = \Sigma = \sum_{j=1}^{k}(s_j-1)u_j u_j^\top+I_p, \text{ with }
s_1\ge \cdots \ge s_k>1,
\]
and the true signal lies in the spiked subspace: $\beta=\sum_{j=1}^{k} r_j u_j$.
\end{assumption}

The next theorem shows that the risk evolution superposes across spike directions.

\begin{theorem}[Risk decomposition for the multi-spike model]\label{thm:multi_spike_risk}
Assume Assumptions~\ref{ass:indpX}--\ref{ass:noise} and Assumption~\ref{ass:multi-spike}. In the asymptotic limit, the deterministic prediction risk satisfies
$\mathcal{R}_{\niter}^{*}=\mathcal{B}_{\niter}^{*}+\mathcal{V}_{\niter}^{*}$ with
\begin{align}
\mathcal{B}_{\niter}^{*}
&= \sum_{j=1}^{k} r_j^{2} s_j\left(1-\left(\frac{s_j}{s_j+\effreg}\right)^{\niter+1}\right)^{2}, \label{eq:multispike_bias}
\end{align}
and
\begin{align}
\mathcal{V}_{\niter}^{*}
&=\frac{\mathcal{V}_{\niter-1}^{*}}{1+\effreg}
+\effreg\sum_{j=1}^{k} r_j^{2}\frac{s_j^{2\niter+1}}{(s_j+\effreg)^{2(\niter+1)}}, \label{eq:multispike_var}
\end{align}
with $\niter\ge 1$ and
$\mathcal{V}_{0}^{*}
=\frac{\noiselev}{\effreg}
+\sum_{j=1}^{k}\frac{\effreg\, r_j^{2} s_j}{(s_j+\effreg)^{2}}$.
\end{theorem}

\textbf{Iterative learning as spectral filtering.} 
Theorem~\ref{thm:multi_spike_risk} shows that different eigendirections are filtered at different rates. The survival of the signal component along direction $u_j$ is governed by the contraction factor $\kappa_j := \frac{s_j}{s_j+\effreg}$.
For strong features ($s_j\gg \effreg$), we have $\kappa_j\approx 1$, and thus the signal is largely preserved across iterations. For weak features ($s_j\approx 1$), we have $\kappa_j\ll 1$, and thus the signal is rapidly forgotten, resembling the isotropic noise.
Thus, {\color{hcolor}\emph{iterative self-training performs soft feature selection: it retains directions that are strong relative to the stochastic noise while filtering out weaker signals together with stochastic noise.}}

\section{Main Results}\label{sec:mainthm}

This section presents our main theoretical results. We focus on the overparameterized regime in which the feature dimension exceeds the sample size at every iteration, i.e., $\ndim>\ndata_{\niter}$. As discussed in Section~\ref{sec:implicitreg}, this is precisely the regime where iterative self-training is non-trivial: each update applies a non-identity projection and induces rich signal--noise dynamics. We derive a precise, non-asymptotic characterization of the estimators $\{\hat{\beta}_{\niter} \}_{0\leq \niter \leq K}$ generated by our iterative self-training scheme.

\subsection{Deterministic equivalents}\label{sec:determ}

The high-dimensional behavior of the ridge estimator at iteration $\niter$ is characterized by two scalar \emph{effective parameters}: an effective regularization level $\effreg_{\niter}$ and an effective noise variance $\effnoise_{\niter}^{2}$. Let $\aspratio_{\ndata,\niter}:=\ndim/\ndata_{\niter}$ denote the aspect ratio at iteration $\niter$. For each $\niter\ge 0$, $(\effreg_{\niter},\effnoise_{\niter}^{2})$ is defined as the unique positive solution to the following fixed-point system.

\begin{definition}[Effective parameters]\label{def:eff}
For each $\niter = 0,1,\ldots,\maxiter$, the pair $(\effreg_{\niter},\effnoise_{\niter}^{2})$ is the unique positive solution to
\begin{align}
\frac{1}{\aspratio_{\ndata,\niter}}
&=\frac{1}{\ndim}\tr\!\Big((\covmat_{\niter}+\effreg_{\niter}I)^{-1}\covmat_{\niter}\Big)
+\frac{\lambda_{\niter}}{\effreg_{\niter}}, \label{eq:fp_tau}\\
\effnoise_{\niter}^{2}
&=
\begin{cases}
\frac{\noiselev+\effreg_{0}^{2}\big\|(\covmat_{0}+\effreg_{0}I)^{-1}\covmat_{0}^{1/2}\beta\big\|_{2}^{2}}
{\frac{\lambda_{0}}{\effreg_{0}}+\frac{\effreg_{0}}{\ndim}\tr\!\Big((\covmat_{0}+\effreg_{0}I)^{-2}\covmat_{0}\Big)}
& \text{when }\niter=0, \\[1.2em]
\frac{\effreg_{\niter}^{2}\big\|(\covmat_{\niter}+\effreg_{\niter}I)^{-1}\covmat_{\niter}^{1/2}\hat{\beta}_{\niter-1}\big\|_{2}^{2}}
{\frac{\lambda_{\niter}}{\effreg_{\niter}}+\frac{\effreg_{\niter}}{\ndim}\tr\!\Big((\covmat_{\niter}+\effreg_{\niter}I)^{-2}\covmat_{\niter}\Big)}
& \text{when }\niter\ge 1.
\end{cases} \label{eq:fp_sigma}
\end{align}
\end{definition}
Existence and uniqueness of a positive solution pair for~\eqref{eq:fp_tau}--\eqref{eq:fp_sigma} follow from~\citet[Proposition~8.1]{han2023distribution}.

\textbf{Deterministic equivalents.}
Define the shrinkage operator
\[
\effprojmat_{\niter}:=(\covmat_{\niter}+\effreg_{\niter}I)^{-1}\covmat_{\niter},
\]
and the cumulative propagator $\effprojmat_{a:b}:=\effprojmat_{a}\effprojmat_{a-1}\cdots \effprojmat_{b}$ for $b\le a$, with the convention $\effprojmat_{a:a+1}:=I$. We define the deterministic prediction risk $\mathcal{R}_{\niter}^{*}$ and the deterministic effective noise variance $\deteffnoise_{\niter}^{2}$ below, and later show that they uniformly approximate the empirical risks and effective noise variances over $\niter$; hence, they serve as deterministic equivalents.

\noindent Define the \emph{deterministic prediction risk} as:
\begin{equation}\label{eq:detrisk}
\begin{split}
\mathcal{R}_{\niter}^{*}
&:=\beta^\top(\effprojmat_{\niter:0}-I)^\top \covmat (\effprojmat_{\niter:0}-I)\beta\\
&\qquad+\textstyle\sum\limits_{h=0}^{\niter}\frac{\deteffnoise_{h}^{2}}{\ndim}
\tr\!\Big(\effprojmat_{\niter:h+1}^{\top}\covmat \effprojmat_{\niter:h+1} \\
&\qquad\qquad\qquad\qquad \times
\effprojmat_{h}(\covmat_{h}+\effreg_{h}I)^{-1}\Big).
\end{split}
\end{equation}
\noindent Define the \emph{deterministic effective noise variance} for $\niter\ge 1$ recursively as:
\begin{align*}
    L_{\niter} \deteffnoise_{\niter}^{2} &= \effreg_{\niter}^{2} \beta^\top \effprojmat_{t:0}^\top (\covmat_{\niter}  + \effreg_{\niter}I)^{-1} \effprojmat_{t-1:0} \beta \\
     &\quad+ \effreg_{\niter}^{2}\sum\limits_{h=0}^{\niter-1} \frac{\deteffnoise_{h}^{2} }{\ndim}\tr \Big( \effprojmat_{\niter:h+1}^\top (\covmat_{\niter} + \effreg_{\niter}I)^{-1}\\
     &\qquad \qquad\qquad\qquad  \times\effprojmat_{\niter-1:h} (\covmat_{h} + \effreg_{h}I)^{-1} \Big)
\end{align*}
with base case $\deteffnoise_{0}^{2}:=\effnoise_{0}^{2}$, where
\[
L_{\niter}:=\frac{\lambda_{\niter}}{\effreg_{\niter}}
+\frac{\effreg_{\niter}}{\ndim}\tr\!\Big((\covmat_{\niter}+\effreg_{\niter}I)^{-2}\covmat_{\niter}\Big).
\]

\textbf{Spectral shrinkage and error propagation.}
The matrix $\effprojmat_{\niter}$ acts as a spectral shrinkage operator. If $\covmat_{\niter}=U\Lambda_{\niter}U^\top$ with eigenvalues $\{\lambda_{\niter,j}\}_{j=1}^{\ndim}$, then $\effprojmat_{\niter}$ shares the same eigenvectors and applies the scalar filter
\[
\phi_{\niter}(\lambda)=\frac{\lambda}{\lambda+\effreg_{\niter}}\in(0,1).
\]
Thus, high-variance directions are largely preserved ($\phi_{\niter}(\lambda)\approx 1$), whereas low-variance directions are attenuated ($\phi_{\niter}(\lambda)\ll 1$). The cumulative propagator $\effprojmat_{\niter:0}=\effprojmat_{\niter:0}=\effprojmat_\niter \effprojmat_{\niter-1}\cdots \effprojmat_0$ therefore implements an iteration-dependent, data-adaptive low-pass filtering in the eigenbasis of the covariances.

With this viewpoint,~\eqref{eq:detrisk} admits the decomposition
\begin{align*}
\mathcal{R}_{\niter}^{*}
&=\underbrace{\beta^\top(\effprojmat_{\niter:0}-I)^\top \covmat (\effprojmat_{\niter:0}-I)\beta}_{\color{hcolor}\text{systematic error / signal forgetting}}\\
&\quad+\underbrace{\textstyle\sum\limits_{h=0}^{\niter}\frac{\deteffnoise_{h}^{2}}{\ndim}
\tr\!\Big(\effprojmat_{\niter:h+1}^{\top}\covmat \effprojmat_{\niter:h+1}\,
\effprojmat_{h}(\covmat_{h}+\effreg_{h}I)^{-1}\Big)}_{\color{hcolor}\text{stochastic error}}.
\end{align*}
The first term increases as the signal is repeatedly transformed by $\effprojmat_{\niter:0}$, while the second term quantifies how injected noise is geometrically filtered through the propagators $\effprojmat_{\niter:h+1}$. {\color{hcolor}\emph{This signal-forgetting--denoising tension underlies the empirically observed $U$-shaped generalization curve and motivates early stopping.}}

\textbf{Concentration around deterministic equivalents.}
We now show that the prediction risk concentrates around its deterministic equivalent $\mathcal{R}_{\niter}^{*}$ in~\eqref{eq:detrisk}, and that the effective noise variance $\effnoise_{\niter}^{2}$ concentrates around $\deteffnoise_{\niter}^{2}$. This justifies using these deterministic quantities as accurate proxies for their empirical counterparts.

\begin{theorem}[Concentration around deterministic equivalents]\label{thm:Gthm}
Let Assumptions~\ref{ass:indpX}--\ref{ass:noise} hold. Suppose there exists $M>0$ such that for all $0\le \niter\le \maxiter$:
(i) $1+1/M<\aspratio_{\ndata,\niter}\le M$,
(ii) $\lambda_{\niter}\le M$ and $\noiselev\le M$, and
(iii) $\|\covmat_{\niter}\|_{2}\le M$ and $\|\covmat_{\niter}^{-1}\|_{2}\le M$.
Fix any $R$ such that $\|\beta\|_{2}\le R$ and $R+1<M$. Then there exists a constant $C=C(M)$ such that for any $\epsilon\in(0,1/2]$ and any $\niter\in\{0,1,\dots,\maxiter\}$,
\begin{align*}
\mathbb{P}\!\left(
\sup_{\lambda_{0},\dots,\lambda_{\niter}\in[0,M]}
\left|\mathcal{R}(\hat{\beta}_{\niter})-\mathcal{R}_{\niter}^{*}\right|>\epsilon
\right)
&\le C\ndim \exp\!\left(-\ndim\epsilon^{4}/C\right),\\
\mathbb{P}\!\left(
\sup_{\lambda_{0},\dots,\lambda_{\niter}\in[0,M]}
\left|\effnoise_{\niter}^{2}-\deteffnoise_{\niter}^{2}\right|>\epsilon
\right)
&\le C\ndim \exp\!\left(-\ndim\epsilon^{4}/C\right).
\end{align*}
\end{theorem}

\textbf{How to use the theorem in practice.} 
Theorem~\ref{thm:Gthm} implies that for large $(\ndata,\ndim)$, one may treat $\mathcal{R}_{\niter}^{*}$ as a reliable proxy for the empirical risk curve and select an early-stopping time by minimizing $\mathcal{R}_{\niter}^{*}$ over $\niter$. In Section~\ref{sec:experiments} we compute $\mathcal{R}_{\niter}^{*}$ numerically by solving the fixed-point equation for $\effreg_{\niter}$ and evaluating the trace terms, and we compare it to Monte Carlo estimates of $\mathcal{R}(\hat{\beta}_{\niter})$.

\subsection{Iterated GCV for risk estimation}\label{sec:gcv}

A central practical challenge in iterative self-training is to select the stopping time and regularization level without access to the true signal or an independent validation set.

Let $A_{\niter}=P_{\niter}\cdots P_{1}$ denote the cumulative projection operator after $\niter$ iterations, where $P_{\niter}$ is the (data-dependent) projection at iteration $\niter$. We propose an \emph{iterated GCV (iGCV)} estimator for the prediction risk. Define the leverage-dependent multiplier
\begin{equation}\label{eq:M_def}
\!\!\!\!\! M_{\niter}(\lambda)
:=\frac{\tr\!\Big(X_{0}A_{\niter}\big(X_{0}^{\top}X_{0}/\ndata_{0}+\lambda I\big)^{-1}X_{0}^{\top}\Big)/\ndata_{0}}
{1-\tr\!\Big(X_{0}\big(X_{0}^{\top}X_{0}/\ndata_{0}+\lambda I\big)^{-1}X_{0}^{\top}\Big)/\ndata_{0}}.
\end{equation}
The iGCV estimator $\widehat{\mathrm{GCV}}^{(\niter)}(\lambda)$ is the average of corrected residuals,
\begin{equation}\label{eq:iterated-gcv}
\frac{1}{\ndata_{0}}\sum_{i=1}^{\ndata_{0}}
\left[
y_i-x_i^\top A_{\niter}\hat{\beta}_{0}
+\bigl(y_i-x_i^\top \hat{\beta}_{0}\bigr)M_{\niter}(\lambda)
\right]^2,
\end{equation}
where $(x_i,y_i)$ are the initial noisy training samples in $\cD_{0}$ and $\hat{\beta}_{0}$ is the initial ridge estimator.

In the ridgeless limit ($\lambda\to 0^{+}$) with $\ndim>\ndata_{0}$, the correction term in~\eqref{eq:iterated-gcv} becomes ill-conditioned. Following~\citet{hastie2022surprises, patil2021uniform}, we define the $\lambda=0$ profile by continuity, replacing the inverse with the Moore--Penrose pseudoinverse so that the iGCV curve remains well defined even in the interpolating regime.

\begin{theorem}[Uniform consistency of iterated GCV]\label{thm:iterated_gcv_consistency}
Suppose Assumptions~\ref{ass:indpX}--\ref{ass:noise} hold. For any fixed $\niter\in\{1,\dots,\maxiter\}$, the iGCV estimator~\eqref{eq:iterated-gcv} satisfies
\[
\sup_{\lambda\in\mathcal{I}}
\left|
\widehat{\mathrm{GCV}}^{(\niter)}(\lambda)-\mathcal{R}(\hat{\beta}_{\niter})-\noiselev
\right|
\xrightarrow{a.s.} 0,
\]
as $\ndata_{0},\ndim\to\infty$ with $\ndim/\ndata_{0}\to\aspratio>1$, where $\mathcal{I}$ is any compact subinterval of $[0,+\infty]$.
\end{theorem}

\begin{figure}[t]
    \centering
    \begin{subfigure}[b]{0.48\linewidth}
        \centering
        \includegraphics[width=\linewidth]{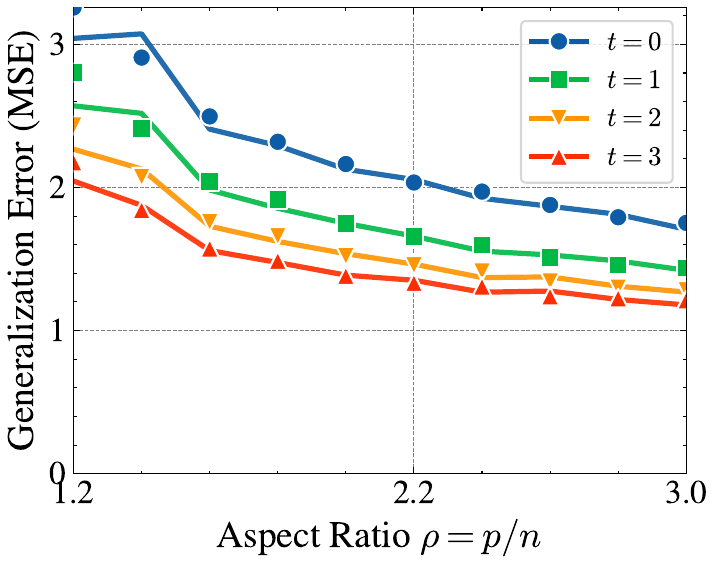}
        \caption{iGCV vs.\ aspect ratio}
        \label{fig:gcv_rho}
    \end{subfigure}
    \hfill
    \begin{subfigure}[b]{0.48\linewidth}
        \centering
        \includegraphics[width=\linewidth]{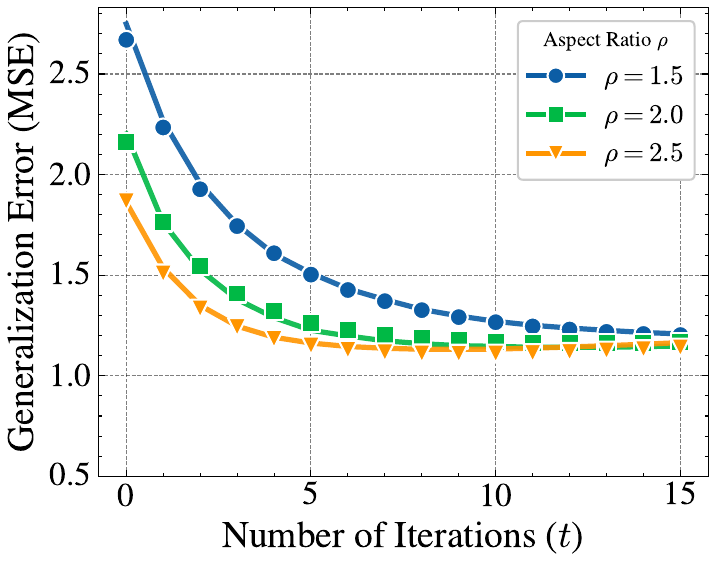}
        \caption{iGCV vs.\ iterations}
        \label{fig:gcv_iterations}
    \end{subfigure}
    \caption{\small \textbf{Validation of the iterated GCV (iGCV) estimator.}
    Solid lines denote the iGCV estimate; markers denote the empirical test risk.
    \textbf{(a)} iGCV tracks the risk across aspect ratios in the overparameterized regime.
    \textbf{(b)} iGCV captures the $U$-shaped risk trajectory over iterations and identifies the optimal early-stopping point.}
    \label{fig:gcv_validation}
\end{figure}

Equation~\eqref{eq:iterated-gcv} enables efficient computation because it only  depends on the initial dataset and $A_t$. In practice, minimizing $\widehat{\mathrm{GCV}}^{(\niter)}(\lambda)$ jointly over $(\niter,\lambda)$ provides a fully data-driven estimate of the bottom of the $U$-shaped risk curve (Figure~\ref{fig:gcv_validation}), yielding a principled stopping rule that prevents signal forgetting from overwhelming denoising.

\section{Numerical Experiments}\label{sec:experiments}

We validate our theoretical predictions and examine the robustness of the denoising--forgetting trade-off beyond the spiked model using both synthetic simulations and real-world experiments on CIFAR-10s~\citep{krizhevsky2009learning}. Detailed experimental settings are provided in Appendix~\ref{sec:RealDatasetting}.


All simulations follow the linear model in~\eqref{eq:lm}. Features $x\in\RR^{\ndim}$ are drawn from a centered Gaussian distribution $x\sim\mathcal{N}(0,\covmat)$, and observation noise satisfies $\varepsilon\sim\mathcal{N}(0,1)$.
For the feature covariance, we consider a diagonal covariance with power-law decay, $\covmat_{ii}=1/i$ for $i=1,\dots,\ndim$. The signal $\beta\in\RR^{\ndim}$ is sparse, supported on the first 10 coordinates: $\beta_i=1$ for $i\le 10$ and $\beta_i=0$ otherwise.

We implement the iterative self-training procedure as follows: (i) Fit an initial estimator $\hat{\beta}_0$ by ridgeless linear regresson on $n=500$ noisy training samples; (ii) For iterations $\niter=1,\dots,50$, draw fresh features $X_{\niter}$ of size $n=500$, generate pseudo-labels $Y_{\niter}=X_{\niter}\hat{\beta}_{\niter-1}$, and fit $\hat{\beta}_{\niter}$ by ordinary least squares on $(X_{\niter},Y_{\niter})$.
We evaluate each iterate by the empirical test MSE on an independent noiseless test set of size $n_{\text{test}}=10{,}000$. This MSE is a Monte Carlo approximation of the prediction risk in~\eqref{eq:riskcon}. Results are averaged over 10 independent trials.

\begin{figure}[t]
    \centering
    \begin{subfigure}[b]{0.48\linewidth}
        \centering
        \includegraphics[width=\linewidth]{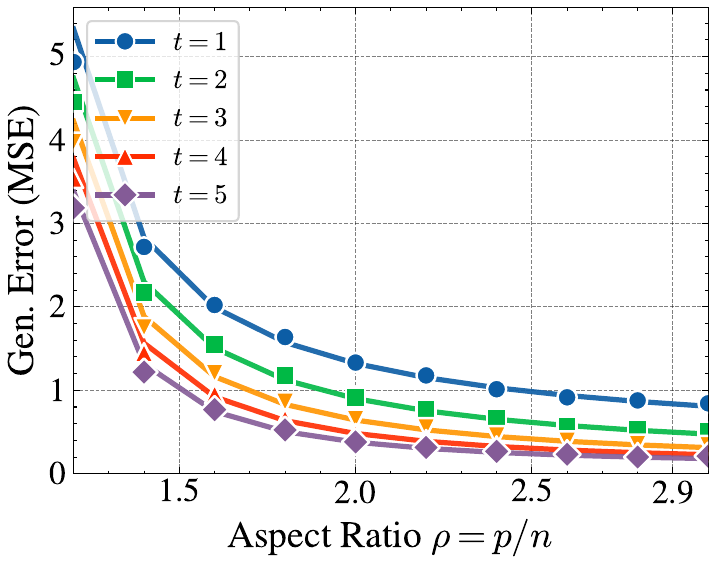}
        \caption{Error vs.\ aspect ratio}
        \label{fig:general_cov_vs_gamma}
    \end{subfigure}
    \hfill
    \begin{subfigure}[b]{0.48\linewidth}
        \centering
        \includegraphics[width=\linewidth]{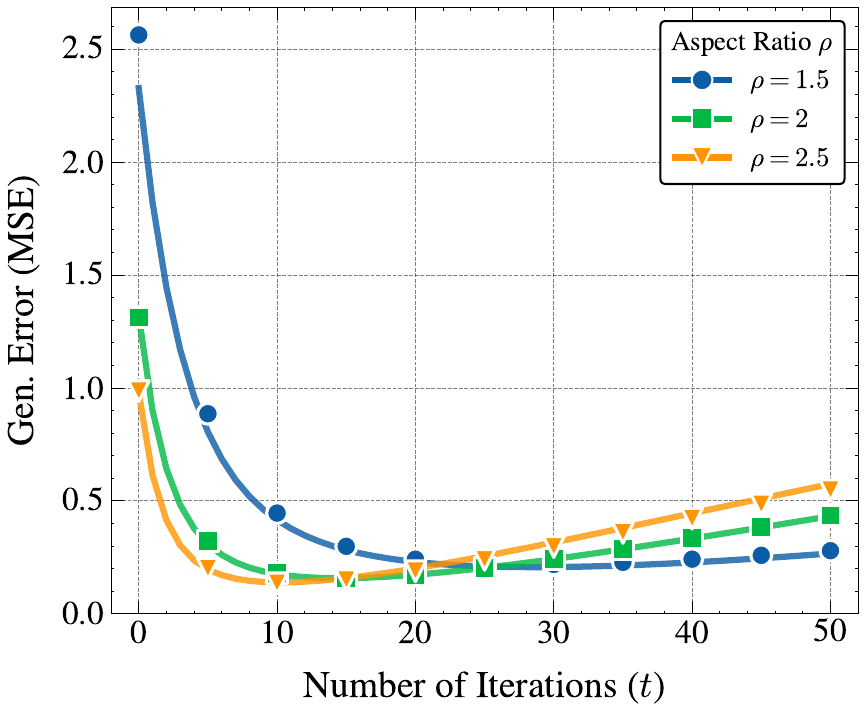}
        \caption{Error vs.\ iterations}
        \label{fig:general_cov_vs_iter}
    \end{subfigure}
    \caption{\small \textbf{General covariance simulations ($\covmat_{ii}=1/i$).}
    Solid lines denote theory and markers denote simulations (averaged over 10 trials).
    \textbf{(a)} Test error versus aspect ratio $\aspratio=\ndim/\ndata$ for fixed iteration counts $\niter$.
    Iteration lowers the error floor and suppresses the double-descent peak.
    \textbf{(b)} Test error versus iteration $\niter$ for fixed aspect ratios $\aspratio$.
    The $U$-shaped curves indicate that the denoising--forgetting trade-off persists under general feature correlations.}
    \label{fig:general_cov_main}
\end{figure}
Our theory matches simulations precisely. Figure~\ref{fig:general_cov_vs_gamma} shows iteration acting as implicit regularization, flattening the double-descent peak and lowering error in the overparameterized regime ($\rho > 1$). Figure~\ref{fig:general_cov_vs_iter} confirms the $U$-shaped trajectory: initial denoising via random projections gives way to signal forgetting, with the optimal $t^*$ determined by the variance-bias balance.


Universality is further evidenced by ResNet-50~\citep{he2016deep} on real dataset CIFAR-10, where the $U$-shaped risk profile persists despite the non-linearities of the deep network (Figure~\ref{fig:real_data_noise}). For moderate noise ($\sigma=0.4$), rapid initial denoising is eventually overtaken by systematic signal drift, whereas higher noise ($\sigma=0.8$) prolongs the variance-reduction phase as predicted by our rate-mismatch analysis. Furthermore, increasing the sample size $n$ enhances spectral separation and accelerates error reduction (Figure~\ref{fig:real_data_size}), reinforcing the conclusion that the competition between stochastic denoising and systematic forgetting is a fundamental structural property of high-dimensional iterative learning that transcends the linear regime.

\begin{figure}[t]
    \centering
    \begin{subfigure}[b]{0.45\linewidth}
        \centering
        \includegraphics[width=\linewidth]{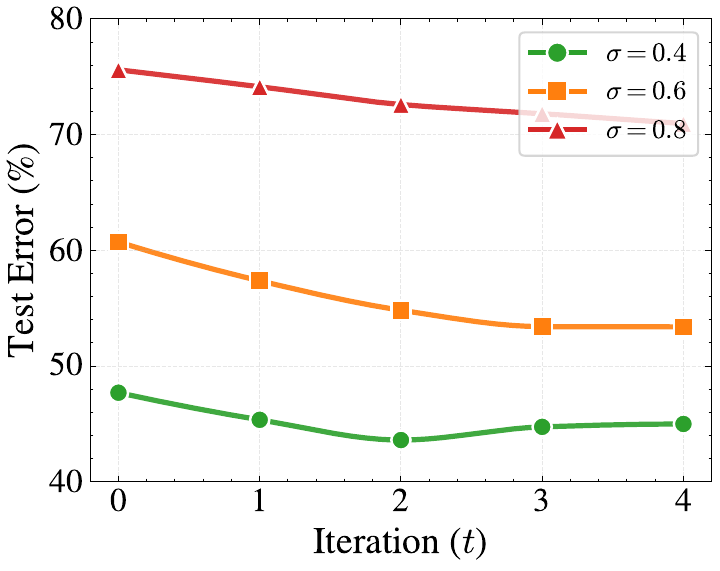}
        \caption{Effect of Label Noise} 
        \label{fig:real_data_noise}
    \end{subfigure}
    \begin{subfigure}[b]{0.45\linewidth}
        \centering
        \includegraphics[width=\linewidth]{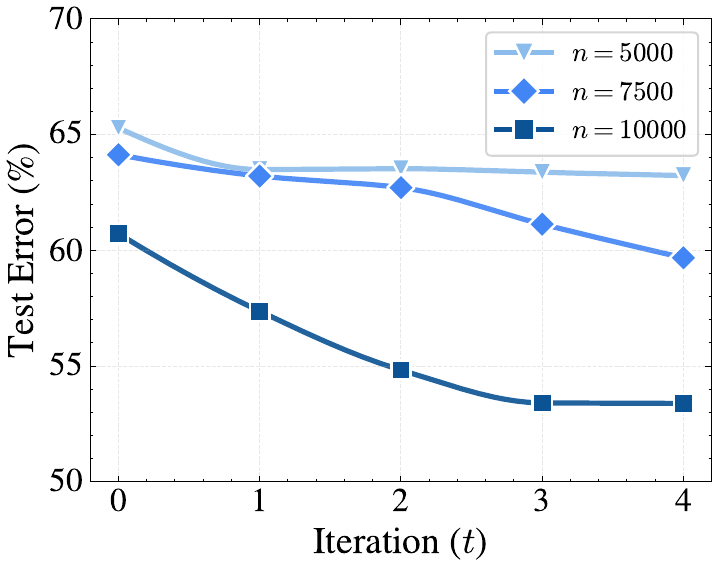}
        \caption{Effect of Sample Size} 
        \label{fig:real_data_size}
    \end{subfigure}

    \caption{\small \textbf{Real-world dynamics on CIFAR-10 using ResNet-50.} 
    Solid lines represent interpolated trajectories, and markers denote experimental test set results.
    \textbf{(a)} Evolution of test error across different initial noise ratios $\sigma$ ($n=10^4$). A characteristic $U$-shaped curve emerges at moderate noise ($\sigma=0.4$), while high noise ($\sigma=0.8$) leads to a more sustained denoising phase.
    \textbf{(b)} Evolution of test error for varying training set sizes $n$ ($\sigma=0.6$). Larger sample sizes accelerate the error reduction and achieve superior final generalization, confirming the scalability of the iterative self-training effect in deep learning settings.}
    \label{fig:real_data_main}
\end{figure}





\section{Conclusions and Discussion}\label{sec:conclusion}

We analyzed iterative self-training in overparameterized linear regression through high-dimensional asymptotics, formalizing the tension between early-stage \emph{denoising} and progressive \emph{signal forgetting}. This trade-off manifests as a $U$-shaped risk curve, driven by an underlying \emph{spectral filtering} mechanism that performs \emph{soft feature selection}. Our deterministic recursions yield precise finite-sample predictions, while our iterated GCV enables optimal early stopping by balancing denoising gains against accumulating systematic drift without requiring external validation.

To bridge this framework with practice, extensions to reused data and filtered pseudo-labels are necessary to address finite-sample constraints and to formalize heuristics that modulate the denoising--forgetting balance. Investigating nonlinear models is also critical to determine whether evolving representations stabilize signal or exacerbate drift. Finally, while simulations show self-training can surpass ridge, a theorem-level comparison—specifying the asymptotic tuning rule for ridge as a function of $\sigval$ is left for future work. Whether analogous deterministic-equivalent descriptions persist in these richer settings remains a fundamental open problem.





\section*{Impact Statement}

This paper presents work whose goal is to advance the field of Machine
Learning. There are many potential societal consequences of our work, none
which we feel must be specifically highlighted here.

\nocite{langley00}

\bibliography{example_paper}
\bibliographystyle{icml2026}

\newpage
\appendix
\onecolumn


\section{Real Data Experiments: ResNet-50 on CIFAR-10}\label{sec:RealDatasetting}

While our theoretical analysis focuses on the analytically tractable overparameterized linear regression, we conjecture that the signal-noise dynamics identified, specifically the trade-off between denoising and signal forgetting, are universal phenomena in iterative self-training. To validate this universality, we conduct experiments using a deep neural network on the CIFAR-10 dataset, a setting that far exceeds the linear assumptions.

\subsection{Experimental Setup.}
We employ a ResNet-50 architecture, a standard overparameterized deep network. We utilize the CIFAR-10 dataset, comprising $32 \times 32$ color images across 10 classes. To closely mimic the finite-sample regime ($p/n > 1$) and control the noise level, we subsample the training data into independent subsets.

The iterative protocol mirrors Algorithm 1 but is adapted for classification. 
First, an {initial teacher} ($t=0$) is trained on a labeled subset of size $n$. To simulate a ``weak'' teacher and introduce initial stochastic error (analogous to $E_0$), we inject symmetric label noise by corrupting a fraction $\eta$ of the training labels.
Subsequently, for iterations $t=1, \dots, \maxiter$, the current model $\widehat{f}_{t-1}$ generates pseudo-labels for a \textit{fresh} subset of data. A new Student model $\widehat{f}_t$ is then trained from scratch on this synthetic data. We fix $\maxiter=4$ and report the test error on the clean CIFAR-10 test set.

\subsection{Results and Discussion.}
\textbf{Effect of Initial Noise Level (Figure \ref{fig:real_data_noise}).}
We first investigate how the strength of the initial corruption affects the trajectory of the test error. We fix the sample size $n=10,000$ and vary the noise ratio $\eta \in \{0.4, 0.6, 0.8\}$.
The results exhibit striking similarities to our theoretical predictions:
\begin{itemize}
    \item \textbf{The $U$-Shaped Curve:} For the moderate noise case ($\eta=0.4$, green line), we observe a clear non-monotonic behavior. The error initially decreases from $47.7\%$ to $43.6\%$ ($t=2$), demonstrating the powerful \textit{denoising effect} ($\mathcal{V}_t^* \downarrow$) of self-training. However, as iterations continue, the error rises to $45.0\%$, confirming the accumulation of \textit{systematic error} ($\mathcal{B}_t^* \uparrow$).
    \item \textbf{Dominance of Denoising:} For the high noise case ($\eta=0.8$, red line), the initial model is extremely weak. Here, the self-training process yields a monotonic decrease. Consistent with our theory, when the stochastic error is dominant, the benefit of noise reduction significantly outweighs the cost of systematic error accumulation in the early stages.
\end{itemize}

\textbf{Effect of Sample Size (Figure \ref{fig:real_data_size}).}
Next, we analyze the impact of the sample size $n$, fixing the noise ratio $\eta=0.6$. Figure \ref{fig:real_data_size} illustrates two key findings. First, \textbf{Weak-to-Strong Generalization} is observed across all sample sizes, as the iterative process successfully distills a better model from the noisy teacher. Second, larger sample sizes ($n=10^4$, dark blue) result in a steeper descent compared to smaller sizes. In the context of our theory, a larger $n$ improves the spectral separation between signal and noise, allowing the implicit projection to more effectively preserve the true signal while filtering out noise.

\section{The Isotropic Covariance Case}\label{sec:iso_cov}

This section presents the risk dynamics in the istropic covariance case. 

\begin{corollary}[Special cases: no signal and isotropic features]\label{cor:noisedecay}
Under the assumptions of Theorem~\ref{thm:spiked_risk}, the following special cases hold:
\begin{enumerate}
\item {No signal ($\siglev^{2}=0$):} $\mathcal{B}_{\niter}^{*}=0$ and
\[
\mathcal{V}_{\niter}^{*}=\frac{\noiselev}{\effreg(1+\effreg)^{\niter}},
\]
corresponding to pure denoising.
\item {Isotropic features ($\sigval=1$):} The risk simplifies to
\[
\mathcal{R}_{\niter}^{*}
=\siglev^{2}\left(1-\frac{1}{(1+\effreg)^{\niter+1}}\right)^{2}
+\frac{\noiselev}{\effreg(1+\effreg)^{\niter}}.
\]
\end{enumerate}
\end{corollary}

In the isotropic setting ($\sigval=1$), $\mathcal{R}_{\niter}^{*}$ is monotone in $\niter$, unlike the spiked covariance case. Indeed,
\[
\mathcal{R}_{\niter+1}^{*}-\mathcal{R}_{\niter}^{*}
=\frac{1}{(1+\effreg)^{\niter+2}}
\left(\frac{\siglev^{2}}{\effreg}-\noiselev(1+\effreg)\right),
\]
whose sign is independent of $\niter$. Iteration helps if and only if the initial error is noise-dominated, i.e.,
$\siglev^{2}/\noiselev < \effreg(1+\effreg)$, confirming that a genuine $U$-shape requires anisotropy ($\sigval>1$).

\begin{proof}
We specialize Theorem 3.2 to the two cases provided in the corollary.

\paragraph{Case 1: No signal ($r^2 = 0$).} 
Substituting $r^2 = 0$ into \eqref{eq:signalerror} and \eqref{eq:stocherror}, the systematic error $\mathcal{B}^*_t$ vanishes for all $t$. The initial condition and recursive step for the stochastic error simplify to:
\begin{equation*}
    \mathcal{V}^*_0 = \frac{\sigma^2}{\tau}, \quad \mathcal{V}^*_t = \frac{\mathcal{V}^*_{t-1}}{1 + \tau}.
\end{equation*}
This is a pure geometric progression. Solving the recurrence yields $V^*_t = \frac{\sigma^2}{\tau(1 + \tau)^t}$, which represents pure denoising.

\paragraph{Case 2: Isotropic features ($s = 1$).} 
When $s = 1$, the contraction factor for the signal becomes $1/(1 + \tau)$, which is identical to the contraction factor of the noise. The systematic error term \eqref{eq:signalerror} becomes:
\begin{equation*}
    \mathcal{B}^*_t = r^2 \left( 1 - \left[ \frac{1}{1 + \tau} \right]^{t+1} \right)^2.
\end{equation*}
The stochastic error recurrence \eqref{eq:stocherror} simplifies to:
\begin{equation*}
    \mathcal{V}^*_t = \frac{\mathcal{V}^*_{t-1}}{1 + \tau} + \frac{\tau r^2}{(1 + \tau)^{2(t+1)}}, \quad \text{with} \quad \mathcal{V}^*_0 = \frac{\sigma^2}{\tau} + \frac{\tau r^2}{(1 + \tau)^2}.
\end{equation*}
In this case, the risk $\mathcal{R}^*_t = \mathcal{B}^*_t + \mathcal{V}^*_t$ is monotonic in $t$. The lack of a "rate mismatch" between signal and noise contraction prevents the emergence of the U-shaped risk curve. Iteration only helps if the initial variance reduction outweighs the bias increase, which depends solely on the initial signal-to-noise ratio.
\end{proof}




\section{Proofs in Section~\ref{sec:toy}}

\subsection{Proof of Theorem~\ref{thm:spiked_risk}: Risk Recursion in the Spiked Model}
\begin{proof}
The proof proceeds by specializing the general deterministic equivalents for risk and effective noise variance from Theorem~\ref{thm:Gthm} to the single-spiked covariance model (Assumption~\ref{ass:spikedmodel}). We analyze the behavior of the key matrices in this model under the high-dimensional limit ($n, p \to \infty$ with $p/n \to \rho > 1$).

\paragraph{Step 1: Asymptotic Analysis of the Effective Regularization, $\tau$}
In the spiked model, the covariance matrix $\Sigma_t = \Sigma = (s - 1)e_1 e_1^\top + I_p$ is constant for all $t$. Consequently, the effective regularization $\tau_t$ will also be a constant, denoted by $\tau$. Its value is determined by the fixed-point equation from Definition~\ref{def:eff} (assuming $\lambda_t \to 0^+$ for ridgeless regression):
\begin{equation*}
    \frac{\ndata}{\ndim} = \frac{1}{\ndim}\tr((\covmat + \effreg I)^{-1} \covmat).
\end{equation*}    
The matrix inside the trace, $\effprojmat: = (\covmat + \effreg I)^{-1} \covmat$ is diagonal in the eigenbasis of $\covmat$. Specifically:
\begin{enumerate}[label=(\roman*)]
    \item The eigenvector $e_{1}$ has an eigenvalue of $\frac{\sigval}{\sigval + \effreg}$.
    \item The $\ndim - 1$ eigenvectors orthogonal to $e_{1}$ have an eigenvalue of $\frac{1}{1 + \effreg}$.
\end{enumerate}
The trace is the sum of these eigenvalues:
\begin{equation*}
     \tr (\effprojmat) = \frac{1}{\ndim}\tr((\covmat + \effreg I)^{-1} \covmat) = \frac{1}{\ndim} \left( \frac{\sigval}{\sigval + \effreg} + (\ndim - 1)\frac{1}{1 + \effreg}\right).
\end{equation*}
In the asymptotic limit where $p \to \infty$, the term involving the spike vanishes, yielding:
\begin{equation*}
   \lim_{\ndim \rightarrow \infty} \tr (\effprojmat) = \lim_{\ndim \rightarrow \infty}\frac{1}{\ndim} \left( \frac{\sigval}{\sigval + \effreg} + (\ndim - 1)\frac{1}{1 + \effreg}\right) = \frac{1}{1 + \effreg}.
\end{equation*}
Substituting this back into the fixed-point equation gives $\frac{1}{\aspratio} = \frac{1}{1 + \effreg}$ gives the constant effective regularization:
\begin{equation*}
    \effreg = \aspratio - 1.
\end{equation*}

\paragraph{Step 2: Derivation of the systematic error, $\mathcal{B}_{\niter}^{*}$}
The systematic error term $\mathcal{B}_{\niter}^{*}$ corresponds to the signal-forgetting component in the general risk formula from Theorem~\ref{thm:Gthm}:
\begin{equation*}
    \mathcal{B}_{\niter}^{*} = \lim_{\ndim \rightarrow \infty}\beta^\top ((\effprojmat_{\niter:0})^\top - I) \covmat (\effprojmat_{\niter:0} - I) \beta = \|\covmat^{1/2}(\effprojmat_{\niter:0} - I)\beta\|_{2}^{2}.
\end{equation*}
In the spiked model, the shrinkage operator $Q$ is constant over iterations, so the cumulative propagator $\effprojmat_{\niter:0} = \effprojmat^{\niter + 1}$. We substitute the aligned signal $\beta = \siglev e_{1}$. Since $e_{1}$ is an eigenvector of $\effprojmat$ with eigenvalue $\frac{\sigval}{\sigval + \effreg}$, we have:
\begin{equation*}
    \covmat^{1/2}(\effprojmat_{\niter:0} - I) \beta = \covmat^{1/2}(\effprojmat^{\niter+1} - I) \siglev e_{1} = \siglev \sqrt{\sigval} \left(1 - \left( \frac{\sigval}{\sigval + \effreg}\right)^{\niter+1} \right) e_{1}.
\end{equation*}
Taking the squared Euclidean norm yields the desired expression for the systematic bias:
\begin{equation*}
    \mathcal{B}_{\niter}^{*} = \siglev^{2} \sigval \left(1 - \left( \frac{\sigval}{\sigval + \effreg}\right)^{\niter+1} \right)^{2}.
\end{equation*}

\paragraph{Step 3: Derivation of the  Stochastic Error Recursion, $\mathcal{V}_{\niter}^{*}$}
The stochastic error term $\mathcal{V}_\niter^{*}$ is defined as the remainder of the total risk after subtracting the systematic error. It is the asymptotic limit of the propagated noise terms in the general risk formula. 

\subparagraph{Initial Condition ($\mathcal{V}_{0}^{*}$):}

\noindent From Theorem~\ref{thm:Gthm}, the initial stochastic error $\mathcal{V}_{0}^{*}$ is determined by the initial deterministic effective noise $\deteffnoise_{0}$. In the asymptotic limit, let's denote this as $\deteffnoise_{\niter}^{*} = \lim_{\ndim \rightarrow \infty}\deteffnoise_{\niter}$ The formula for $\deteffnoise_{0}$ is $L_{0} \deteffnoise_{0}^{2} = \noiselev + \effreg \|(\covmat + \effreg I)^{-1} \covmat^{1/2} \beta\|_{2}^{2}$. We analyze these components in the limit $p \to \infty$:
\begin{itemize}
    \item The squared norm term evaluates to:
    \begin{equation*}
        \|(\covmat + \effreg I)^{-1} \covmat^{1/2} \siglev e_{1}\|_{2}^{2} = \siglev^{2} \sigval \left\|\frac{1}{\sigval + \effreg} e_{1} \right\|_{2}^{2} = \frac{\siglev^{2} \sigval}{(\sigval + \effreg)^{2}}.
    \end{equation*}
    \item The term $L_{0}$:
    \begin{equation*}
        \lim_{\ndim \rightarrow \infty} L_{0} = \lim_{\ndim \rightarrow \infty} \frac{\effreg}{\ndim} \tr((\covmat + \effreg I)^{-2} \covmat) = \frac{\effreg}{(1 + \effreg)^{2}}
    \end{equation*}
\end{itemize}
Substituting these limits, we obtain:
\begin{equation*}
    \frac{\effreg}{(1 + \effreg)^{2}} \deteffnoise_{0}^{*2} = \noiselev + \frac{\effreg \siglev^{2} \sigval}{(\sigval + \effreg)^{2}}.
\end{equation*}
Finally, for $\mathcal{V}_{0}^{*}$, we have:
\begin{equation*}
    \mathcal{V}_{0}^{*} = \lim_{\ndim \rightarrow \infty} \frac{\deteffnoise_{0}^{2}}{\ndim} \tr((\covmat + \effreg I)^{-2} \covmat) = \frac{\noiselev}{\effreg} + \frac{\siglev^{2} \sigval}{(\sigval + \effreg)^{2}}
\end{equation*}
This matches the formula for $\mathcal{V}_{0}^{*}$ in the theorem statement.

\subparagraph{Step 3.1: Asymptotic Recursion for Deterministic Noise $\deteffnoise_\niter^*$.}
We evaluate the general recursion for $D_t^2$ from Theorem 4.3 in the high-dimensional limit:
\begin{equation*}
    L_{\niter} \deteffnoise_{\niter}^{2} = \effreg_{\niter}^{2} \beta^\top \effprojmat_{t-1:0}^\top \effprojmat_{\niter}(\covmat_{\niter} + \effreg_{\niter}I)^{-1} \effprojmat_{t-1:0} \beta + \effreg_{\niter}^{2}\sum\limits_{h=0}^{\niter-1} \frac{\deteffnoise_{h}^{2} }{\ndim}\tr\left( \effprojmat_{\niter-1:h+1}^\top \effprojmat_{\niter}(\covmat_{\niter} + \effreg_{\niter}I)^{-1} \effprojmat_{\niter-1:h+1} \effprojmat_{h}(\covmat_{h} + \effreg_{h}I)^{-1} \right).
\end{equation*}
For constant spiked matrices, the recursion simplifies to:
\begin{equation*}
L^* \deteffnoise_{\niter}^{*2} = \lim_{\ndim \rightarrow \infty}\effreg^{2} \beta^\top (\effprojmat^\niter)^\top \effprojmat(\covmat + \effreg I)^{-1} \effprojmat^{\niter} \beta + \effreg^2 \sum_{h=0}^{\niter-1} \deteffnoise_h^{*2} \lim_{\ndim\to\infty} \frac{1}{\ndim} \tr\left((\effprojmat^{\niter-h-1})^\top \effprojmat(\covmat + \effreg I)^{-1} \effprojmat^{\niter-h-1} \effprojmat(\covmat + \effreg I)^{-1}\right).
\end{equation*}
We evaluate each component's limit:
\begin{itemize}
    \item The bias term is exactly: $\siglev^{2} \left(\frac{\sigval}{\sigval+\effreg}\right)^{2\niter} \frac{\sigval}{(\sigval+\effreg)^2}$.
    \item The trace term is dominated by the $\ndim-1$ isotropic eigenvectors: $\frac{1}{(1+\effreg)^{2(k-h+1)}}$.
\end{itemize}
Substituting $L^* = \frac{\effreg}{(1+\effreg)^2}$, we obtain:
\begin{equation} \label{eq:proof_Dk_recursion}
    \frac{\effreg}{(1+\effreg)^2} \deteffnoise_\niter^{*2} = \effreg^2 \siglev^{2} \frac{\sigval^{2\niter+1}}{(\sigval+\effreg)^{2\niter+2}} + \effreg^2 \sum_{h=0}^{\niter-1} \frac{\deteffnoise_h^{*2}}{(1+\effreg)^{2(\niter-h+1)}}.
\end{equation}

\subparagraph{Step 3.2: Relating $\mathcal{V}_\niter^*$ to $\deteffnoise_\niter^*$.}The stochastic error $\mathcal{V}_\niter^*$ is the asymptotic limit of the sum-over-noise term. For the spiked model:
\begin{equation} \label{eq:proof_Vk_sum}
    \mathcal{V}_\niter^* = \lim_{\ndim\to\infty} \sum_{h=0}^{\niter} \frac{\deteffnoise_h^{2}}{\ndim} \tr\left( (\effprojmat^{\niter-h})^\top \covmat \effprojmat^{\niter-h} \effprojmat(\covmat+\effreg I)^{-1} \right) = \sum_{h=0}^{\niter} \frac{\deteffnoise_h^{*2}}{(1+\effreg)^{2(\niter-h+1)}}.
\end{equation}
Separating the $t$-th term from the sum gives:
\begin{align}
    \mathcal{V}_\niter^* &= \frac{\deteffnoise_\niter^{*2}}{(1+\effreg)^2} + \sum_{h=0}^{\niter-1} \frac{\deteffnoise_h^{*2}}{(1+\effreg)^{2(\niter-h + 1)}} \nonumber \\
    &= \frac{\deteffnoise_\niter^{*2}}{(1+\effreg)^2} + \frac{1}{(1+\effreg)^2} \sum_{h=0}^{\niter-1} \frac{\deteffnoise_h^{*2}}{(1+\effreg)^{2(\niter-h)}} \nonumber \\
    &= \frac{\deteffnoise_\niter^{*2}}{(1+\effreg)^2} + \frac{\mathcal{V}_{\niter-1}^*}{(1+\effreg)^2}. \label{eq:proof_Vk_Vk-1_relation}
\end{align}
This identity shows how the total stochastic error evolves by adding the newly generated noise at step $\niter$ to the propagated noise from the previous step.

\subparagraph{Step 3.3: Final Derivation.}
Our goal is to express $\mathcal{V}_\niter^*$ solely in terms of $\mathcal{V}_{\niter-1}^*$. We can achieve this by substituting the recursion for $\deteffnoise_\niter^*$ \eqref{eq:proof_Dk_recursion} into our expression for $\mathcal{V}_\niter^*$ \eqref{eq:proof_Vk_Vk-1_relation}.
From \eqref{eq:proof_Dk_recursion}, we have:
\begin{equation*}
    \frac{\deteffnoise_\niter^{*2}}{(1+\effreg)^2}  = \effreg \siglev^{2} \frac{\sigval^{2\niter+1}}{(\sigval+\effreg)^{2\niter+2}} + \effreg \sum_{h=0}^{\niter-1} \frac{\deteffnoise_h^{*2}}{(1+\effreg)^{2(\niter-h+1)}}.
\end{equation*}
From \eqref{eq:proof_Vk_sum}, we have:
\begin{equation*}
    \frac{\deteffnoise_\niter^{*2}}{(1+\effreg)^2}  = \effreg \siglev^{2} \frac{\sigval^{2\niter+1}}{(\sigval+\effreg)^{2\niter+2}} + \frac{\effreg \mathcal{V}_{\niter-1}^*}{(1+\effreg)^2}.
\end{equation*}
Now, substitute this into the relation for $\mathcal{V}_\niter^*$ from \eqref{eq:proof_Vk_Vk-1_relation}:
\begin{align*}
    \mathcal{V}_\niter^* &= \left(\effreg \siglev^{2} \frac{\sigval^{2\niter+1}}{(\sigval+\effreg)^{2\niter+2}} + \frac{\effreg \mathcal{V}_{\niter-1}^*}{(1+\effreg)^2} \right) + \frac{\mathcal{V}_{\niter-1}^*}{(1+\effreg)^2} \\
    &= \effreg \siglev^{2} \frac{\sigval^{2\niter+1}}{(\sigval+\effreg)^{2\niter+2}} +  \frac{\mathcal{V}_{\niter-1}^*}{1+\effreg}.
\end{align*}
This completes the proof.
\end{proof}

\subsection{Proof of Corollary~\ref{lm:stochreduct}}
\begin{proof}
From the recursion derived in Theorem~\ref{thm:spiked_risk}, the stochastic error $\mathcal{V}_{\niter}^{*}$ follows a linear inhomogeneous recurrence:
\begin{equation}\label{eq:Vt_recurrence_proof}
    \mathcal{V}_{\niter}^{*} = \frac{1}{1 + \effreg} \mathcal{V}_{\niter-1}^{*} + C_\niter, \quad \text{where} \quad C_\niter = \effreg \siglev^{2} \frac{\sigval^{2\niter+1}}{(\sigval+\effreg)^{2\niter+2}}.
\end{equation}
Since the aspect ratio $\aspratio > 1$, the effective regularization $\effreg = \aspratio - 1$ is strictly positive. This implies that the contraction coefficient $1/(1 + \tau)$ lies in the interval $(0, 1)$. To prove that $V^*_t \to 0$ as $t \to \infty$, we analyze the driving term $C_t$:
\begin{equation*}
    C_t = \frac{\tau r^2}{s} \left( \frac{s}{s + \tau} \right)^{2t+2}.
\end{equation*}
Since $s > 1$ and $\tau > 0$, the ratio $\kappa^2 = (\frac{s}{s+\tau})^2$ is strictly less than $1$. Thus, $C_t$ is a geometric sequence that decays exponentially to zero as $t \to \infty$.

Unrolling the recurrence \eqref{eq:Vt_recurrence_proof} from $t$ down to $0$, we obtain:
\begin{equation*}
    V^*_t = \left( \frac{1}{1 + \tau} \right)^t V^*_0 + \sum_{h=1}^t \left( \frac{1}{1 + \tau} \right)^{t-h} C_h.
\end{equation*}
The first term $\left( \frac{1}{1 + \tau} \right)^t V^*_0$ vanishes as $t \to \infty$. The second term is a convolution of two sequences that both decay exponentially to zero. By the properties of convergent sequences (specifically, the Cauchy product or a direct application of the squeeze theorem), we conclude:
\begin{equation*}
    \lim_{t \to \infty} V^*_t = 0.
\end{equation*}
This demonstrates that iterative self-training consistently attenuates the stochastic error inherited from the initial fit and the random projections.
\end{proof}



\subsection{Proof of Corollary~\ref{lm:strongsp}}
\begin{proof}
We analyze the limits of the risk components as the spike strength $s \to \infty$ for a fixed iteration $t$.

\paragraph{Part 1: Limit of the systematic error ($\mathcal{B}_{\niter}^*$)}
The expression for the systematic error is:
\begin{align*}
\mathcal{B}_{\niter}^* &= \siglev^{2} \sigval \left(1 - \left(\frac{\sigval}{\sigval + \effreg} \right)^{\niter + 1} \right)^{2}.
\end{align*}

Using a first-order Taylor expansion for large $s$:
\begin{equation*}
    \left( \frac{s}{s + \tau} \right)^{t+1} = \left( 1 + \frac{\tau}{s} \right)^{-(t+1)} = 1 - \frac{(t+1)\tau}{s} + O(s^{-2}).
\end{equation*}
Substituting this into the bias formula:
\begin{equation*}
    B^*_t = r^2 s \left( \frac{(t+1)\tau}{s} + O(s^{-2}) \right)^2 = \frac{r^2 (t+1)^2 \tau^2}{s} + O(s^{-2}).
\end{equation*}
Taking the limit $s \to \infty$, we find $B^*_t \to 0$. Physically, this means that for a sufficiently strong spike, the iterative projection barely erodes the signal aligned with that direction.

\paragraph{Part 1: Limit of the systematic error ($\mathcal{V}_{\niter}^*$)}
We proceed by induction on $t$.

\textit{Base Case ($t=0$):} From the initial condition in Theorem~\ref{thm:spiked_risk}:
\begin{equation*}
    \lim_{\sigval \rightarrow \infty} \mathcal{V}_{0}^* = \frac{\noiselev}{\effreg} + \lim_{\sigval \rightarrow \infty} \frac{\effreg \sigval}{(\sigval + \effreg)^{2}} \siglev^{2} = \frac{\noiselev}{\effreg}.
\end{equation*}
This matches the target formula $\frac{\sigma^2}{\tau (1+\tau)^t}$ for $t=0$.

\textit{Inductive Step:} Assume the result holds for some $\niter - 1$, i.e., $\mathcal{V}_{\niter - 1}^* = \frac{\noiselev}{\effreg(1 + \effreg)^{\niter-1}}$. For iteration $t$, the recurrence gives:
\begin{equation*}
    \lim_{s \to \infty} V^*_t = \lim_{s \to \infty} \left( \frac{V^*_{t-1}}{1 + \tau} + \tau r^2 \frac{s^{2t+1}}{(s + \tau)^{2(t+1)}} \right).
\end{equation*}
Analyze the second term:
\begin{equation*}
    \lim_{s \to \infty} \frac{s^{2t+1}}{(s + \tau)^{2t+2}} = \lim_{s \to \infty} \frac{1}{s} \left( \frac{1}{1 + \tau/s} \right)^{2t+2} = 0 \cdot 1 = 0.
\end{equation*}
Thus, the limit satisfies:
\begin{equation*}
    \lim_{s \to \infty} V^*_t = \frac{1}{1 + \tau} \left( \frac{\sigma^2}{\tau (1+\tau)^{t-1}} \right) + 0 = \frac{\sigma^2}{\tau (1 + \tau)^t}.
\end{equation*}
This completes the induction. The result shows that in the strong-spike regime, the systematic error vanishes while the noise is filtered at a rate of $(1+\tau)^{-1}$ per iteration.
\end{proof}


\section{Proof of Theorem~\ref{thm:Gthm}}
\subsection{The Distribution of Ridge(less) Regression Estimator}

This section characterizes the distribution of the Ridge and ridgeless estimator, $\hat{\beta}_{\lambda} $ within the linear model defined in~\eqref{eq:lm}. We introduce a connection to a corresponding Ridge estimator in a simpler Gaussian sequence model. For a given covariance matrix $\covmat$
, coefficient vector $\beta$, and noise level $\effnoise > 0$, this sequence model is defined as:
\begin{equation}\label{eq:seqmd}
    y^{\text{seq}}_{(\covmat, \beta)} (\effnoise) = \covmat^{1/2} \beta + \frac{\effnoise g}{\sqrt{\ndim}}, \quad g \sim \mathcal{N}(0, I).
\end{equation}
It is crucial to note that this effective noise level, $\effnoise$, is distinct from the noise variance, $\noiselev$, in the original linear model. Instead, $\effnoise$ is endogenously determined as part of the solution to a system of fixed-point equations that will be introduced in the subsequent theorem.

The Ridge estimator, denoted as $\hat{\beta}^{\text{seq}}_{(\covmat, \beta)} (\effnoise; \effreg)$, with a regularization parameter $\effreg \geq 0$ in the Gaussian sequence model~\eqref{eq:seqmd}, is the solution to the following optimization problem:
\begin{align}\notag
\hat{\beta}^{\text{seq}}_{(\covmat, \beta)} (\effnoise; \effreg)
&:=  \argmin_{b \in \RR^{\ndim}} \left\{ \frac{1}{2\ndata}\|\covmat^{1/2}b - y^{\text{seq}}_{(\covmat, \beta)} (\effnoise)\|_2^2 + \frac{\effreg}{2} \|b\|_{2}^{2} \right\} \\ \label{eq:seqest}
&= (\covmat+ \effreg I)^{-1} \covmat^{1/2}\left( \covmat^{1/2} \beta + \frac{\effnoise g}{\sqrt{\ndim}} \right).
\end{align}

The following theorem, which extends Theorem 2.3 in \citep{han2023distribution}, characterizes the distribution of $\hat{\beta}_{\lambda}$ in terms of $\hat{\beta}^{\text{seq}}_{(\covmat, \beta)}$. Our extension specifically addresses the case where the noise variance, $\noiselev = 0$, can be zero.

\begin{theorem}[Distributional characterization~\citep{han2023distribution, ildiz2025high}]\label{thm:distri}
    Assume Assumptions~\ref{ass:Gauss}-~\ref{ass:noise} hold. Further, assume that for some constant $M > 0$, we have 
, and that $1 + 1/M < \aspratio_{\ndata} \leq M$ and $\noiselev$, $\| \covmat\|_{2}$, $\| \covmat^{-1}\|_{2} \leq M$. Then, for any $L$-Lipschitz function $f: \RR^{\ndim} \rightarrow \RR$ with $L < L(M)$, there exists a constant $C = C(M)$ such that
    for any $\epsilon \in (0, 1/2]$, the following holds:
    \begin{equation}
         \mathbb{P} \left( \sup_{\lambda \in [0, M]} \left| f(\hat{\beta}_{\lambda})  - \EE_{g} f(\hat{\beta}^{\text{seq}}_{(\covmat, \beta)} (\effnoise; \effreg))\right| > \epsilon\right) \leq C\ndim e^{-\ndim\epsilon^{4}/C},
    \end{equation}
    where $R < M$, and $(\effnoise,\effreg)$ is the unique positive solution of the following wing fixed point equations:
    \begin{equation}
    \left\{
    \begin{aligned}
        \aspratio_{\ndata} &= \frac{1}{\ndim}\tr\left((\covmat + \effreg I )^{-1} \covmat \right) + \frac{\lambda}{\effreg},\\
        \effnoise^{2} &= \frac{\noiselev + \effreg^{2} \| (\covmat + \effreg I )^{-1} \covmat^{1/2} \beta \|_{2}^{2}}{\frac{\lambda}{\effreg} + \frac{\effreg}{\ndim} \tr\left((\covmat + \effreg I )^{-2} \covmat \right)}.
    \end{aligned}
    \right. 
    \end{equation}
\end{theorem}
\begin{proof}
    The primary distinction between this result and Theorem 2.3 in~\citep{han2023distribution} is our accommodation for a zero noise level, $\noiselev = 0$.  In the original proof, a non-zero noise level is crucial for establishing a lower bound on the effective noise, specifically $\effnoise^2 > \frac{1}{C}$ for some constant $C > 0$. To extend the characterization to the noiseless setting, we have removed the requirement for the bound to hold uniformly over $\beta$.
    
    We now consider the case where $\noiselev = 0$. The expression for the effective noise becomes:
    \begin{equation*}
        \effnoise^{2} = \frac{ \effreg^{2} \| (\covmat + \effreg I )^{-1} \covmat^{1/2} \beta \|_{2}^{2}}{\frac{\lambda}{\effreg} + \frac{\effreg}{\ndim} \tr\left((\covmat + \effreg I )^{-2} \covmat \right)} \geq \frac{\effreg \| (\covmat + \effreg I )^{-1} \covmat^{1/2} \beta \|_{2}^{2}}{\| \covmat^{-1}\|_{\text{op}}/(\| \covmat^{-1}\|_{\text{op}} + \effreg)^2}.
    \end{equation*}
    From this expression, it is evident that $\effnoise = 0$ if and only if $\beta = 0$ (since we have assumed $\| \covmat^{-1}\|_{\text{op}} \leq M$).

    Let us examine the estimators in this specific scenario where $\noiselev = 0$ and $\beta = 0$. The ridge estimator is given by:
    \begin{align*}
        \hat{\beta}_{\lambda} &= (X^\T  X + n\lambda I)^{-1} {X}^\T Y\\
        &= (X^\T  X + n\lambda I)^{-1} {X}^\T X\beta + (X^\T  X + n\lambda I)^{-1} {X}^\T \varepsilon
    \end{align*}
    Since $\beta = 0$ and $\noiselev = 0$ implies $\varepsilon = 0$, we have $\hat{\beta}_{\lambda} = 0$.
    
    Similarly, from equation~\eqref{eq:seqest}, the ridge estimator in the sequence model is:
    \begin{align*}
        \hat{\beta}^{\text{seq}}_{(\covmat, 0)} (0; \effreg) &= (\covmat+ \effreg I)^{-1} \covmat^{1/2}\left( \covmat^{1/2} \cdot 0 + \frac{0 \cdot g}{\sqrt{\ndata}} \right) = 0.
    \end{align*}
    In the case where $\noiselev = 0$ and $\beta = 0$, both estimators are zero. Therefore, the statement of the theorem holds, as the difference between the functions of the estimators is zero. For the case where $\beta \neq 0$, the effective noise $\effnoise$ remains positive, and the arguments from the proof of Theorem 2.3 in~\citep{han2023distribution} continue to apply. This completes the extension of the theorem to the noiseless setting.
\end{proof} 

\subsection{Proof of Theorem~\ref{thm:Gthm}}
As a preliminary step for the proof of Theorem~\ref{thm:Gthm}, we must first establish uniform high-probability bounds on the norms of the estimators generated by our iterative scheme. The following lemma provides these crucial bounds, extending the result of Proposition 11 in~\citep{ildiz2025high} from a single-step to our multi-step procedure.

\begin{lemma}[Uniform High-Probability Bounds on Estimator Norms]\label{lm:estimatorbound}
    Suppose the assumptions of Theorem~\ref{thm:Gthm} hold. Specifically, assume that for some constant $M > 1$, $1/M \leq \aspratio_{\ndata}, \noiselev \leq M$, and for all iterations$\niter \leq \maxiter$, $\|\covmat_{\niter}\|_{2}, \|\covmat_{\niter}^{-1}\|_{2} \leq M$. Then, for any constant  $c > 0$, there exists an event $\mathcal{E} = \mathcal{E}(M,c)$ with $\mathbb{P}(\mathcal{E}^{c}) \leq C e^{-\ndim/C}$ for some $C = C(M ,c)$, on which the following bounds hold simultaneously for all $0 \leq \niter \leq \maxiter$ and for any matrix $A$ with operator norm $\|A\|_{2} \leq 1$:
    \begin{enumerate}
    \item $\|\hat{\beta}_{\niter}\|_{2} \leq \|\beta\|_{2} + c.$
    \item $\|A\hat{\beta}_{\niter} - \beta\|_{2} \leq \|\beta\|_{2} + c.$
    \end{enumerate}
\end{lemma}
\begin{proof}
    The proof proceeds by induction on the iteration index $\niter$.
    
    \noindent\textbf{Base Case $(\niter = 0)$:}
    
    \noindent The initial estimator is \begin{equation*}
        \hat{\beta}_{0} = (X_{0}^{T}X_{0} + \ndata \lambda_{0}I)^{+} X_{0}^{T} Y_{0} = (X_{0}^{T}X_{0} + \ndata_{0} \lambda_{0}I)^{+} X_{0}^{T} (X_{0}\beta + E)
    \end{equation*}
    Let $P_{0} = (X_{0}^{T}X_{0} + \ndata_{0} \lambda_{0}I)^{+} X_{0}^{T} X_{0}$, we can write:
    \begin{equation*}
        \hat{\beta}_{0} = P_{0} \beta + (X_{0}^{T}X_{0} + \ndata_{0} \lambda_{0}I)^{+} X_{0}^{T} E.
    \end{equation*}
    Proposition 11 in~\citep{ildiz2025high} establishes that on the high-probability event $\mathcal{E}$, the stochastic term is bounded:
    \begin{equation}\label{eq:lm1}
        \| (X_{0}^{T}X_{0})^{+} X_{0}^{T} E\|_{2} \leq c.
    \end{equation}
    \begin{enumerate}

        \item \textbf{Bound on $\|\hat{\beta}_{0}\|_{2}$:} Using the triangle inequality and the fact that $\|P_{0}\|_{2} \leq 1$:
        \begin{equation*}
            \|\hat{\beta}_{0}\|_{2} = \|P_{0} \beta\|_{2} + \|(X_{0}^{T}X_{0} + \ndata_{0} \lambda_{0}I)^{+} X_{0}^{T} E\|_{2} \leq \|P_{0}\|_{2} \|\beta\|_{2} + c \leq \|\beta\|_{2} + c.
        \end{equation*}
        \item \textbf{Bound on $\|A\hat{\beta}_{0} - \beta\|_{2}$:} We again use the triangle inequality:
        \begin{align*}
            \|A\hat{\beta}_{0} - \beta\|_{2}^{2} = \|A(P_{0}\beta - \beta +  (X_{0}^{T}X_{0} + \ndata_{0} \lambda_{0}I)^{+} X_{0}^{T} E)\|_{2}\\
            \leq \|A(P_{0}\beta - \beta)\|_{2} + \|A(X_{0}^{T}X_{0} + \ndata_{0} \lambda_{0}I)^{+} X_{0}^{T} E)\|_{2}.
        \end{align*}
        The second term is bounded by $\|A\|_{2} \|X_{0}^{T}X_{0} + \ndata_{0} \lambda_{0}I)^{+} X_{0}^{T} E\|_{2} \leq c$. For any $A$ with $\|A\|_{2} \leq 1$, the first term, representing the bias, can be bounded as $\|(AP_{0}- I) \beta)\|_{2} \leq \|\beta\|_{2}$. Combining these gives:
        \begin{equation*}
            \|A\hat{\beta}_{0} - \beta\|_{2} \leq \|\beta\|_{2} + c.
        \end{equation*}
    \end{enumerate}
    Thus, the base case holds on the event $\mathcal{E}$.
    
    \noindent\textbf{Inductive Step:}
    
    \noindent Assume that for some $\niter - 1 \leq \maxiter$, the bounds hold for $\hat{\beta}_{\niter-1}$. We need to show they also hold for $\hat{\beta}_{\niter}$. For $\niter \geq 1$, the estimator update is:
    \begin{equation*}
        \hat{\beta}_{\niter} = (X_{\niter}^{T}X_{\niter} + \ndata_{\niter} \lambda_{\niter} I)^{+} X_{\niter}^{T}Y_{\niter} = (X_{\niter}^{T}X_{\niter} + \ndata_{\niter} \lambda_{\niter} I)^{+} X_{\niter}^{T} X_{\niter}\hat{\beta}_{\niter - 1}.
    \end{equation*}
    Let $M_{\niter}:= (X_{\niter}^{T}X_{\niter} + \ndata_{\niter} \lambda_{\niter} I)^{+} X_{\niter}^{T} X_{\niter}.$ This matrix is positive semi-definite, and $\|M_{\niter}\|_{2} \leq 1$.
    \begin{enumerate}
        \item \textbf{Bound on $\|\hat{\beta}_{\niter}\|_{2}$:} 
        \begin{equation*}
            \|\hat{\beta}_{\niter}\|_{2} = \|M_{\niter}\hat{\beta}_{\niter-1}\|_{2} \leq \|M_{\niter}\|_{2} \|\hat{\beta}_{\niter-1}\|_{2} \leq \|\hat{\beta}_{\niter-1}\|_{2}.
        \end{equation*}
        By the inductive hypothesis, $\|\hat{\beta}_{\niter-1}\|_{2} \leq \|\beta\|_{2} + c$, so the bound holds for $\|\hat{\beta}_{\niter}\|_{2}$.

        \item \textbf{Bound on $\|A\hat{\beta}_{\niter} - \beta\|_{2}$:}
        \begin{equation*}
            \|A\hat{\beta}_{\niter} - \beta\|_{2} = \|AM_{\niter}\hat{\beta}_{\niter-1} - \beta\|_{2}.
        \end{equation*}
        By the inductive hypothesis, since $\|AM_{\niter}\|_{2} \leq 1$ this quantity is bounded:
        \begin{equation*}
            \|AM_{\niter}\hat{\beta}_{\niter-1} - \beta\|_{2} \leq \|\beta\|_{2} + c.
        \end{equation*}
    \end{enumerate}
    This completes the inductive step. The claims hold for all $\niter \leq \maxiter$ on the event $\mathcal{E}$.       
\end{proof}

The following lemma establishes the Lipschitz continuity of two key functions that appear in the proof of Theorem~\ref{thm:Gthm}. This property is essential for applying concentration inequalities to our estimators.

\begin{lemma}[Lipschitz Properties]\label{lm:Lipschitz}
    Under the assumptions of Theorem~\ref{thm:Gthm}, and on the high-probability event $\mathcal{E}$ from Lemma~\ref{lm:estimatorbound}, the following holds for any matrix $A$ with $\|A\|_{2} \leq 1$ and all $0 \leq k \leq \maxiter$. 

    Let the functions $f_{1}, f_{2,k}: \RR^{\ndim} \rightarrow \RR$ be defined as:
    \begin{enumerate}
        \item $f_{1}(b) = \| \covmat^{1/2} (Ab- \beta) \|_{2}^{2}$. 
        \item $f_{2,k}(b) = \| (\covmat_{k} + \effreg_{k}I)^{-1} \covmat_{k}^{1/2} A b \|_{2}^{2}$. 
    \end{enumerate}
    Then, for any estimator $\hat{\beta}_{\niter}$ from our scheme $(0 \leq \niter \leq \maxiter)$, the gradients of these functions are uniformly bounded on $\mathcal{E}$ by:
    \begin{equation*}
        \| \nabla f_{1}(\hat{\beta}_{\niter})\|_{2}, \| \nabla f_{2,k}(\hat{\beta}_{\niter})\|_{2} \leq L,
    \end{equation*}
    where $L = 2M(\|\beta\|_{2} + c)$. This implies that the functions are Lipschitz continuous in the region where the estimators lie.
\end{lemma}
\begin{proof}
We prove the bound for the gradient norm of each function when evaluated at an estimator $\hat{\beta}_{\niter}$ (for any $0 \leq \niter \leq \maxiter$) on the event $\mathcal{E}$:

\noindent \textbf{1. Analysis of $f_{1}(b)$:}

\noindent The function is $f_{1}(b) = (Ab- \beta)^\top \covmat (Ab- \beta)$. Using standard vector calculus, its gradient with respect to $b$ is:
    \begin{equation*}
        \nabla f_{1}(b) = 2 A^\top \covmat (Ab- \beta).
    \end{equation*}
    We bound the Euclidean norm of this gradient when evaluated at $b = \hat{\beta}_{\niter}$:
    \begin{align*}
        \| \nabla f_{1}(\hat{\beta}_{\niter}) \|_{2} &= \|2A^\top \covmat (A\hat{\beta}_{\niter} - \beta)\|_{2} \\
        &\leq 2\|A\|_{2} \|\covmat\|_{2}  \|A\hat{\beta}_{\niter} - \beta\|_{2} \\
        &\leq 2M(\|\beta\|_{2} + c),
    \end{align*}
    where last inequality follow from Lemma~\ref{lm:estimatorbound}. This establishes the bound $L$.

\noindent \textbf{1. Analysis of $f_{2,k}(b)$:}

\noindent  The gradient with respect to $b$ is
\begin{equation*}
    \nabla f_{2,k}(b) = 2A (\covmat_{k} + \effreg_{k}I)^{-2} \covmat_{k} A \hat{\beta}_{\niter}.
\end{equation*}
Now, we bound its norm when evaluated at $b = \hat{\beta}_{\niter}$:
    \begin{align*}
        \| \nabla f_{2}(\hat{\beta}_{\niter}) \|_{2} &= \| 2A (\covmat_{k} + \effreg_{k}I)^{-2} \covmat_{k} A \hat{\beta}_{\niter}\| \\
        &\leq 2\|A\|_{2}^{2} \| (\covmat_{k} + \effreg_{k}I)^{-2} \covmat_{k}\|_{2} \|\hat{\beta}_{\niter}\|_{2}\\
        &\leq 2(\|\beta\|_{2} + c)(\lambda_{k, 0} \vee \lambda_{k, \ndim}^{-1})\\
        &\leq 2M(\|\beta\|_{2} + c),
    \end{align*}
    where $\lambda_{k, 0}$ and $\lambda_{k, \ndim}$ are the largest and smallest eigenvalue of $\covmat_{k}$, respectively. This establishes the bound $L$.
\end{proof}

\begin{proof}[Proof of Theorem~\ref{thm:Gthm}]
    The proof proceeds by induction on the iteration index $\niter$. The core of the argument is the repeated application of the distributional characterization from Theorem~\ref{thm:distri}. This allows us to replace the random estimator $\hat{\beta}_{\niter}$ 
 with its simpler Gaussian sequence model equivalent, $\hat{\beta}_{\niter}^{\text{seq}}$, when computing the expectation of any Lipschitz function. By fully "unrolling" the recursive definition of $\hat{\beta}_{\niter}^{\text{seq}}$, we can express it as a function of the initial signal $\beta$ and the history of independent Gaussian noise vectors. This allows for a direct calculation of the deterministic equivalents.

 We begin by defining the sequence-model estimators using the propagator notation $\effprojmat_{\niter} = (\covmat_{\niter} + \effreg_{\niter} I)^{-1} \covmat_{\niter}$:
 \begin{align*}
     \hat{\beta}_{0}^{\text{seq}} &= \effprojmat_{0} \beta + \frac{\effnoise_{0} }{\sqrt{\ndim}} (\covmat_{0} + \effreg_{0} I)^{-1} \covmat_{0}^{1/2} g_{0}\\
     \hat{\beta}_{\niter}^{\text{seq}} &= \effprojmat_{\niter} \hat{\beta}_{\niter-1}+ \frac{\effnoise_{\niter}}{\sqrt{\ndim}} (\covmat_{\niter} + \effreg_{\niter} I)^{-1} \covmat_{\niter}^{1/2} g_{\niter},
 \end{align*}
where $\{ g_{\niter}\}_{0 \leq \niter \leq \maxiter}$ are independent standard Gaussian vectors.

\noindent \textbf{Part 1: Deterministic Equivalent of the Prediction Risk}

\noindent The prediction risk $\mathcal{R}(\hat{\beta}_{\niter}) = \| \covmat^{1/2} (\hat{\beta}_{\niter} - \beta) \|_{2}^{2}$ is a Lipschitz function of $\hat{\beta}_{\niter}$ (Lemma~\ref{lm:Lipschitz}). Its deterministic equivalent $\mathcal{R}_{\niter}$ is therefore $\EE_{g_{\niter}} [\mathcal{R}(\hat{\beta}_{\niter}^{\text{seq}})]$.

\textbf{1. Base Case ($\niter = 0$):}

The deterministic risk $\mathcal{R}_{0}$is the expectation of the risk of the sequence-model estimator:
\begin{align*}
\mathcal{R}_0 &= \EE_{g_0} \left[ \|\covmat^{1/2}(\hat{\beta}_0^{\text{seq}} - \beta)\|_{2}^{2} \right] \\
&= \EE_{g_0} \left[ \left\|\covmat^{1/2} \left( (\effprojmat_0 - I)\beta + \frac{\effnoise_0}{\sqrt{\ndim}} (\covmat_0 + \effreg_{0} I)^{-1} \covmat_{0}^{1/2} g_0 \right) \right\|_{2}^{2} \right].
\end{align*}
We expand the squared norm. The cross-term vanishes because $\EE[ g_{0}] = 0$. We are left with the sum of the squared norms of the deterministic bias and the expected noise term:

\begin{align*}
\mathcal{R}_0 &= \|\covmat^{1/2}(\effprojmat_0 - I)\beta\|_{2}^{2}+ \frac{\effnoise_0^2}{\ndim} \EE_{g_0} \left[ \| \covmat^{1/2} (\covmat + \effreg_{0} I)^{-1} \covmat^{1/2} g_0 \|_{2}^{2} \right] \\
&= \beta^\top (\effprojmat_0 - I)^\top \covmat (\effprojmat_0 - I) \beta + \frac{\effnoise_0^2}{\ndim} \tr\left(\covmat (\covmat_{0} + \effreg_0 I)^{-1} \covmat_{0} (\covmat_0 + \effreg_0 I)^{-1} \right) \\
&= \beta^\top (\effprojmat_0 - I)^\top \covmat (\effprojmat_0 - I) \beta + \frac{\effnoise_0^2}{\ndim} \tr\left(\covmat \effprojmat_0 (\covmat_0 + \effreg_0 I)^{-1} \right).
\end{align*}
This matches the formula for $\mathcal{R}_{\niter}$ in the theorem statement for $\niter = 0$.

\textbf{2. Inductive Step:}

\noindent This proof is more direct by unrolling the full recursion. By repeatedly substituting the definition, we express $\hat{\beta}_{\niter}^{\text{seq}}$ in terms of $\beta$ and the full history of noise $\{ g_{h}\}_{0 \leq h \leq \niter}$:
\begin{equation*}
    \hat{\beta}_{\niter}^{\text{seq}} = \effprojmat_{\niter:0}\beta + \sum_{h=0}^{\niter} \frac{\effnoise_h^2}{\sqrt{\ndim}} \effprojmat_{\niter:h+1} (\covmat_h + \effreg_h I)^{-1} \covmat_h^{1/2} g_h
\end{equation*}
The deterministic risk $\mathcal{R}_{\niter}$ is the expected risk of this fully-unrolled estimator, where the expectation is over the entire history of independent noise vectors:
\begin{equation*}
    \mathcal{R}_{\niter} = \EE_{\{ g_{h}\}_{0 \leq h \leq \niter}} \left[ \left\| \covmat^{1/2} \left((\effprojmat_{\niter:0} - I)\beta + \sum_{h=0}^{\niter} \frac{\effnoise_h^2}{\sqrt{\ndim}} \effprojmat_{\niter:h+1} (\covmat_h + \effreg_h I)^{-1} \covmat_h^{1/2} g_h \right)\right\|_{2}^{2}\right].
\end{equation*}
Because the noise vectors $\{ g_{h}\}_{0 \leq h \leq \niter}$ are independent and have zero mean, all cross-terms of the form $\EE[g_{h}^\top A g_{j}$ for $h \neq j$ are zero. The expectation of the squared norm becomes the squared norm of the bias term plus the sum of the expected squared norms of the noise terms:
\begin{align*}
\mathcal{R}_\niter &= \|\covmat^{1/2}(\effprojmat_{\niter:0} - I)\beta \|_{2}^{2} + \sum_{h=0}^{\niter} \frac{\effnoise_h^2}{\ndim} \EE_{g_h} \left[ \|\covmat^{1/2} \effprojmat_{\niter:h+1} (\covmat_h + \effreg_h I)^{-1} \covmat_h^{1/2} g_h \|_{2}^{2} \right] \\
&= \beta^\top (\effprojmat_{k:0} - I)^\top \covmat (\effprojmat_{k:0} - I)\beta + \sum_{h=0}^{\niter} \frac{\effnoise_h^2}{\ndim} \tr\left(\effprojmat_{k:h+1}^\top \covmat \effprojmat_{k:h+1} Q_h (\covmat_h + \effreg_h I)^{-1}\right).
\end{align*}
This confirms the expression for $\mathcal{R}_\niter$ for all $\niter \geq 0$.

\noindent \textbf{Part 2: Deterministic Recursion for Effective Noise}   

\noindent The same logic applies to the recursion for the deterministic effective noise, $\deteffnoise_{\niter}$. The numerator in the fixed-point equation for $1 \leq \niter \leq \maxiter$ is \begin{equation*}
    N_{\niter} = \effreg_{\niter}^{2} \| (\covmat_\niter + \effreg_\niter I)^{-1} \covmat_\niter^{1/2} \hat{\beta}_{\niter-1}\|_{2}^{2}.
\end{equation*} 
This is a Lipschitz function of $\hat{\beta}_{\niter-1}$. Its deterministic equivalent, which we use to define $\deteffnoise_{\niter}$, is found by replacing $\hat{\beta}_{\niter-1}$ with $\hat{\beta}_{\niter-1}^{\text{seq}}$ and taking the full expectation.
\begin{equation*}
    L_{\niter}\deteffnoise_{\niter} = \EE_{\{ g_{h}\}_{0 \leq h \leq \niter}} \left[\effreg_{\niter}^{2} \| (\covmat_\niter + \effreg_\niter I)^{-1} \covmat_\niter^{1/2} \hat{\beta}_{\niter-1}^{\text{seq}}\|_{2}^{2} \right].
\end{equation*}
We use the unrolled expression for $\hat{\beta}_{\niter-1}^{\text{seq}}$:
\begin{equation*}
    \hat{\beta}_{\niter-1}^{\text{seq}} = \effprojmat_{\niter-1:0}\beta + \sum_{h=0}^{\niter-1} \frac{\effnoise_h^2}{\sqrt{\ndim}} \effprojmat_{\niter-1:h+1} (\covmat_h + \effreg_h I)^{-1} \covmat_h^{1/2} g_h.
\end{equation*}
Substituting this into the expectation, and again using the independence of the $\{ g_{h}\}_{0 \leq h \leq \niter}$ to eliminate cross-terms, we get:
\begin{align*}
L_\niter \deteffnoise_{\niter} = \effreg_\niter^2 \|(\covmat_\niter + \effreg_\niter I)^{-1} \covmat_\niter^{1/2} \effprojmat_{\niter-1:0} \beta \|_{2}^{2} + \sum_{h=0}^{\niter-1} \frac{\effnoise_h^2}{\ndim} \EE_{g_h} \left[ \effreg_\niter^2 \|(\covmat_\niter + \effreg_\niter I)^{-1} \covmat_\niter^{1/2} \effprojmat_{\niter-1:h+1} (\covmat_h + \effreg_h I)^{-1} \covmat_h^{1/2} g_h \|_{2}^{2} \right].
\end{align*}
Replacing the random $\effnoise_h^2$ with their deterministic equivalents $\deteffnoise_{h}$ and simplifying the expressions within the norm and trace gives:
\begin{equation*}
    L_{\niter} \deteffnoise_{\niter}^{2} = \effreg_{\niter}^{2} \beta^\top \effprojmat_{t-1:0}^\top \effprojmat_{\niter}(\covmat_{\niter} + \effreg_{\niter}I)^{-1} \effprojmat_{t-1:0} \beta + \effreg_{\niter}^{2}\sum\limits_{h=0}^{\niter-1} \frac{\deteffnoise_{h}^{2} }{\ndim}\tr\left( \effprojmat_{\niter-1:h+1}^\top \effprojmat_{\niter}(\covmat_{\niter} + \effreg_{\niter}I)^{-1} \effprojmat_{\niter-1:h+1} \effprojmat_{h}(\covmat_{h} + \effreg_{h}I)^{-1} \right)
\end{equation*}
This completes the proof of the recursive formula for the deterministic effective noise. The concentration bounds follow directly from applying Theorem~\ref{thm:distri} at each step.
\end{proof}

\section{Proof of Iterated GCV}\label{proof:gcv}

This section provides the formal proof for the consistency of the Iterated Generalized Cross-Validation (iGCV) estimator. Our proof extends the uniform GCV analysis of \citet{patil2021uniform} to the iterative self-training setting. The primary technical challenge arises from the fact that our estimator at iteration $t$ is defined as $\hat{\beta}_t(\lambda) = A_t \hat{\beta}_0(\lambda)$, where $A_t = P_t P_{t-1} \cdots P_1$ is a cumulative projection operator that introduces a directional dependency not present in standard ridge regression.

\paragraph{Notation and Setup.} Throughout this section, we denote the initial sample covariance matrix as $\hat{\Sigma} := \frac{1}{n} X_0^\top X_0$ and the population covariance as $\Sigma$. We operate in the proportional asymptotic regime where $n, p \to \infty$ such that $p/n \to \rho \in (1, \infty)$. We assume the cumulative projection matrix $A_t$ is independent of the initial training data $(X_0, Y_0)$, which holds by Assumption 2.1. We further define the iterative residual operator $\mathcal{O}_t(\lambda)$ and the leverage multiplier $M_t(\lambda)$ as specified in the main text.

\paragraph{Proof Strategy.}
The proof is structured as follows:
\begin{enumerate}
    \item \textbf{Bias-Variance Decomposition:} We first decompose both the prediction risk $\mathcal{R}(\hat{\beta}_t; \beta)$ and the iGCV estimator $\widehat{\text{GCV}}^{(t)}(\lambda)$ into their respective bias and variance components. This isolates the deterministic signal-dependent terms from the stochastic noise-driven terms.
    \item \textbf{Asymptotic Equivalence:} We then demonstrate that the bias of the iGCV estimator asymptotically tracks the prediction bias, and the iGCV variance tracks the prediction variance (up to the irreducible noise level $\sigma^2$). This is achieved by showing that the leverage-dependent multiplier $M_t(\lambda)$ effectively corrects the in-sample bias induced by overparameterization.
    \item \textbf{Uniformity on $\lambda$:} Finally, we extend the point-wise convergence results to hold uniformly over $\lambda$ within any compact interval. This ensures that the GCV profile can be reliably used for hyperparameter optimization and identifying the optimal stopping time $t$.
\end{enumerate}

\begin{lemma}[Asymptotic Risk Decomposition] \label{lemma:risk_decomp}
For any fixed iteration $t \geq 0$ and regularization $\lambda \geq 0$, the prediction risk $\mathcal{R}(\hat{\beta}_t(\lambda); \beta)$ of the iterative self-training estimator converges almost surely:
\begin{equation}
    \lim_{n,p \to \infty} \left| \mathcal{R}(\hat{\beta}_t(\lambda); \beta) - \left( \mathcal{B}_t(\lambda) + \mathcal{V}_t(\lambda) \right) \right| = 0,,
\end{equation}
where the bias and variance components are defined as:
\begin{align}
    \mathcal{B}_t(\lambda) &= \beta^\top \left( A_t (\hat{\Sigma} + \lambda I_p)^{-1} \hat{\Sigma} - I_p \right)^\top \Sigma \left( A_t (\hat{\Sigma} + \lambda I_p)^{-1} \hat{\Sigma} - I_p \right) \beta, \\
    \mathcal{V}_t(\lambda) &= \frac{\sigma^2}{n} \mathrm{tr} \left( \Sigma A_t (\hat{\Sigma} + \lambda I_p)^{-1} \hat{\Sigma} (\hat{\Sigma} + \lambda I_p)^{-1} A_t^\top \right).
\end{align}
\end{lemma}

\begin{proof}
To derive the decomposition, we first expand the estimation error $\hat{\beta}_t - \beta$. Recall that $\hat{\beta}_t = A_t \hat{\beta}_0$, where $\hat{\beta}_0 = (\hat{\Sigma} + \lambda I_p)^{-1} \frac{1}{n} X_0^\top Y_0$ is the initial ridge estimator. Substituting $Y_0 = X_0 \beta + E_0$, we can write:
\begin{equation*}
    \hat{\beta}_t - \beta = \underbrace{\left( A_t (\hat{\Sigma} + \lambda I_p)^{-1} \hat{\Sigma} - I_p \right) \beta}_{\Delta_b} + \underbrace{A_t (\hat{\Sigma} + \lambda I_p)^{-1} \frac{1}{n} X_0^\top E_0}_{\Delta_v},
\end{equation*}
where $\Delta_b$ represents the bias error vector and $\Delta_v$ represents the stochastic error vector. The prediction risk is then given by the quadratic form:
\begin{equation} \label{eq:risk_expansion}
    \mathcal{R}(\hat{\beta}_t; \beta) = (\Delta_b + \Delta_v)^\top \Sigma (\Delta_b + \Delta_v) = \Delta_b^\top \Sigma \Delta_b + \Delta_v^\top \Sigma \Delta_v + 2 \Delta_b^\top \Sigma \Delta_v.
\end{equation}
We now analyze the three terms in the limit $n, p \to \infty$:
\begin{enumerate}
    \item \textbf{Bias Term:} The first term $\Delta_b^\top \Sigma \Delta_b$ corresponds exactly to $\mathcal{B}_t(\lambda)$. 
    \item \textbf{Variance Term:} For the second term, we take the expectation with respect to the noise $E_0 \sim \mathcal{N}(0, \sigma^2 I_n)$. Using the cyclic property of the trace and $\hat{\Sigma} = \frac{1}{n} X_0^\top X_0$:
    \begin{align*}
        \mathbb{E}_{E_0}[\Delta_v^\top \Sigma \Delta_v] &= \frac{1}{n^2} \mathrm{tr} \left( \Sigma A_t (\hat{\Sigma} + \lambda I)^{-1} X_0^\top \mathbb{E}[E_0 E_0^\top] X_0 (\hat{\Sigma} + \lambda I)^{-1} A_t^\top \right) \\
        &= \frac{\sigma^2}{n} \mathrm{tr} \left( \Sigma A_t (\hat{\Sigma} + \lambda I)^{-1} \hat{\Sigma} (\hat{\Sigma} + \lambda I)^{-1} A_t^\top \right).
    \end{align*}
    This concentration follows from standard random matrix theory results for quadratic forms in the overparameterized regime.
    \item \textbf{Cross Term:} The term $2 \Delta_b^\top \Sigma \Delta_v$ is linear in $E_0$. Since $\mathbb{E}[E_0] = 0$ and $E_0$ is independent of the signal $\beta$ and all feature matrices $\{X_\tau\}_{\tau=0}^t$, this term converges to zero almost surely as $n \to \infty$ by the law of large numbers (see e.g., Lemma 5.1 in \citep{patil2021uniform}).
\end{enumerate}
Summing these components yields the desired decomposition $\mathcal{B}_t(\lambda) + \mathcal{V}_t(\lambda)$.
\end{proof}

\begin{lemma}[Asymptotic iGCV Decomposition] \label{lemma:igcv_decomp_formal}
Define the iterative bias operator $\mathcal{O}_t(\lambda) \in \mathbb{R}^{p \times p}$ and the iterative variance operator $\mathbf{\Xi}_t(\lambda) \in \mathbb{R}^{n \times n}$ as follows:
\begin{align}
    \mathcal{O}_t(\lambda) &:= \left( I_p - A_t (\hat{\Sigma} + \lambda I_p)^{-1} \hat{\Sigma} \right) + M_t(\lambda) \left( I_p - (\hat{\Sigma} + \lambda I_p)^{-1} \hat{\Sigma} \right), \label{eq:Ot_def_formal} \\
    \mathbf{\Xi}_t(\lambda) &:= \left( I_n - \frac{1}{n} X_0 A_t (\hat{\Sigma} + \lambda I_p)^{-1} X_0^\top \right) + M_t(\lambda) \left( I_n - \frac{1}{n} X_0 (\hat{\Sigma} + \lambda I_p)^{-1} X_0^\top \right). \label{eq:Xit_def_formal}
\end{align}
The iterated GCV estimator $\widehat{\text{GCV}}^{(t)}(\lambda)$ converges almost surely to:
\begin{equation}
    \lim_{n,p \to \infty} \left| \widehat{\text{GCV}}^{(t)}(\lambda) - \left( \text{GCV}_b^{(t)}(\lambda) + \text{GCV}_v^{(t)}(\lambda) \right) \right| = 0,
\end{equation}
where the bias and variance components are expressed as:
\begin{align}
    \text{GCV}_b^{(t)}(\lambda) &= \beta^\top \mathcal{O}_t(\lambda)^\top \hat{\Sigma} \, \mathcal{O}_t(\lambda) \beta, \\
    \text{GCV}_v^{(t)}(\lambda) &= \frac{\sigma^2}{n} \mathrm{tr} \left( \mathbf{\Xi}_t(\lambda)^\top \mathbf{\Xi}_t(\lambda) \right).
\end{align}
\end{lemma}

\begin{proof}
The proof relies on decomposing the corrected residual vector into signal and noise contributions. The iGCV estimator is defined as the normalized squared norm of the corrected residuals:
\begin{equation} \label{eq:igcv_base_formal}
    \widehat{\text{GCV}}^{(t)}(\lambda) = \frac{1}{n} \left\| (Y_0 - X_0 \hat{\beta}_t) + M_t(\lambda) (Y_0 - X_0 \hat{\beta}_0) \right\|_2^2.
\end{equation}
By substituting the data generating process $Y_0 = X_0 \beta + E_0$ and the iterative estimators $\hat{\beta}_0 = (\hat{\Sigma} + \lambda I)^{-1} \frac{1}{n} X_0^\top Y_0$ and $\hat{\beta}_t = A_t \hat{\beta}_0$, we expand the residual vector: \begin{align*}
    \text{Res}_{\text{corr}} &= (X_0 \beta + E_0 - X_0 A_t \hat{\beta}_0) + M_t(\lambda) (X_0 \beta + E_0 - X_0 \hat{\beta}_0) \\
    &= \underbrace{X_0 \left[ (I_p - A_t (\hat{\Sigma} + \lambda I)^{-1} \hat{\Sigma}) \beta + M_t (I_p - (\hat{\Sigma} + \lambda I)^{-1} \hat{\Sigma}) \beta \right]}_{\text{Signal Part: } \xi_b} \\
    &\quad + \underbrace{\left[ (I_n - \frac{1}{n} X_0 A_t (\hat{\Sigma} + \lambda I)^{-1} X_0^\top) E_0 + M_t (I_n - \frac{1}{n} X_0 (\hat{\Sigma} + \lambda I)^{-1} X_0^\top) E_0 \right]}_{\text{Noise Part: } \xi_v}.
\end{align*}

\paragraph{Signal Contribution ($\xi_b$).}
Setting the noise $E_0 = 0$, the signal residual at iteration $t$ corrected by the $t=0$ residual is:
\begin{align*}
    \xi_b &= X_0 \left( I_p - A_t (\hat{\Sigma} + \lambda I)^{-1} \hat{\Sigma} \right) \beta + M_t X_0 \left( I_p - (\hat{\Sigma} + \lambda I)^{-1} \hat{\Sigma} \right) \beta \\
    &= X_0 \left[ \left( I_p - A_t (\hat{\Sigma} + \lambda I)^{-1} \hat{\Sigma} \right) + M_t \left( I_p - (\hat{\Sigma} + \lambda I)^{-1} \hat{\Sigma} \right) \right] \beta.
\end{align*}
Using the definition of $\mathcal{O}_t(\lambda)$ in \eqref{eq:Ot_def_formal}, we have $\xi_b = X_0 \mathcal{O}_t(\lambda) \beta$. Its contribution to the GCV profile is:
\begin{equation*}
    \frac{1}{n} \| \xi_b \|_2^2 = \beta^\top \mathcal{O}_t^\top \left( \frac{1}{n} X_0^\top X_0 \right) \mathcal{O}_t \beta = \beta^\top \mathcal{O}_t^\top \hat{\Sigma} \mathcal{O}_t \beta,
\end{equation*}
which matches $\text{GCV}_b^{(t)}(\lambda)$.

\paragraph{Noise Contribution ($\xi_v$).}
Setting the signal $\beta = 0$, the noise-driven residuals are expressed in the observation space $\mathbb{R}^n$. The corrected noise residual is:
\begin{align*}
    \xi_v &= \left(I_n - \frac{1}{n} X_0 A_t (\hat{\Sigma} + \lambda I)^{-1} X_0^\top \right) E_0 + M_t(\lambda) \left(I_n - \frac{1}{n} X_0 (\hat{\Sigma} + \lambda I)^{-1} X_0^\top \right) E_0 = \mathbf{\Xi}_t(\lambda) E_0.
\end{align*}
Taking the expectation over the initial noise $E_0 \sim \mathcal{N}(0, \sigma^2 I_n)$, we apply the concentration of quadratic forms for high-dimensional random vectors:
\begin{equation*}
    \frac{1}{n} \| \xi_v \|_2^2 = \frac{1}{n} E_0^\top \mathbf{\Xi}_t^\top \mathbf{\Xi}_t E_0 \xrightarrow{a.s.} \frac{\sigma^2}{n} \tr(\mathbf{\Xi}_t^\top \mathbf{\Xi}_t).
\end{equation*}
This trace represents the noise variance component $\text{GCV}_v^{(t)}(\lambda)$. Note that $\mathbf{\Xi}_t$ is not necessarily symmetric since $A_t$ may not commute with $\hat{\Sigma}$; thus the product $\mathbf{\Xi}_t \mathbf{\Xi}_t^\top$ is required for the general case.

\paragraph{Cross Term.}
The cross-term $\frac{2}{n} \xi_b^\top \xi_v = \frac{2}{n} \beta^\top \mathcal{O}_t^\top X_0^\top \mathbf{\Xi}_t E_0$ is linear in $E_0$. Since $\mathbb{E}[E_0] = 0$ and $E_0$ is independent of $(X_0, A_t)$, this term converges to zero almost surely as $n \to \infty$ by the law of large numbers. Combining these components yields the desired decomposition.
\end{proof}

\begin{lemma}[Asymptotic Equivalence of iGCV Bias] \label{lemma:igcv_bias_equiv}
For any fixed iteration $t \ge 1$ and $\lambda \geq 0$, the bias component of the iGCV estimator satisfies:
\begin{equation}
    \left| \text{GCV}_b^{(t)}(\lambda) - \mathcal{B}_t(\lambda) \right| \xrightarrow{a.s.} 0.
\end{equation}
\end{lemma}

\begin{proof}
Recall that the GCV bias component is defined as $\text{GCV}_b^{(t)}(\lambda) = \frac{1}{n} \sum_{i = 1}^{n} (x_{0,i}^\top \mathcal{O}_t(\lambda) \beta)^2$. We analyze the convergence of the term $\mathcal{O}_t(\lambda)^\top x_{0,i}$. By definition of the operator $\mathcal{O}_t$, we split it into two terms:
\begin{equation} \label{eq:split}
    \mathcal{O}_t(\lambda)^\top x_{0,i} = \underbrace{\left( I_p -  \hat{\Sigma} (\hat{\Sigma} + \lambda I_p)^{-1}  A_t\right) x_{0,i}}_{\text{Term I}} + \underbrace{M_t(\lambda) \left( I_p - \hat{\Sigma} (\hat{\Sigma} + \lambda I_p)^{-1}  \right) x_{0,i}}_{\text{Term II}}.
\end{equation}

\paragraph{Expansion of Term I.} 
Substituting $\hat{\Sigma} = \hat{\Sigma}_{-i} + x_{0,i}x_{0,i}^\top/n$ and applying the Sherman-Morrison-Woodbury formula to $(\hat{\Sigma} + \lambda I_p)^{-1}$, we have:
\begin{align*}
    \text{Term I} &= x_{0,i} - (\hat{\Sigma}_{-i} + x_{0,i}x_{0,i}^\top/n)(\hat{\Sigma}_{-i} + x_{0,i}x_{0,i}^\top/n + \lambda I_p)^{-1} A_t x_{0,i} \\
    &= x_{0,i} - (\hat{\Sigma}_{-i} + x_{0,i}x_{0,i}^\top/n) \left[ (\hat{\Sigma}_{-i} + \lambda I_p)^{-1} - \frac{(\hat{\Sigma}_{-i} + \lambda I_p)^{-1} x_{0,i}x_{0,i}^\top (\hat{\Sigma}_{-i} + \lambda I_p)^{-1}/n}{1 + x_{0,i}^\top (\hat{\Sigma}_{-i} + \lambda I_p)^{-1}x_{0,i}/n} \right] A_t x_{0,i}.
\end{align*}
Expanding the product of the terms inside the square bracket:
\begin{align*}
    \text{Term I} &= x_{0,i} - \left( \hat{\Sigma}_{-i} (\hat{\Sigma}_{-i} + \lambda I_p)^{-1} -  \frac{\hat{\Sigma}_{-i}(\hat{\Sigma}_{-i} + \lambda I_p)^{-1} x_{0,i}x_{0,i}^\top (\hat{\Sigma}_{-i} + \lambda I_p)^{-1}/n}{1 + x_{0,i}^\top (\hat{\Sigma}_{-i} + \lambda I_p)^{-1}x_{0,i}/n} \right. \\
    &\quad \left. + \frac{x_{0,i}x_{0,i}^\top(\hat{\Sigma}_{-i} + \lambda I_p)^{-1}}{n} - \frac{x_{0,i}x_{0,i}^\top(\hat{\Sigma}_{-i} + \lambda I_p)^{-1} x_{0,i}x_{0,i}^\top (\hat{\Sigma}_{-i} + \lambda I_p)^{-1}}{n^2 (1 + x_{0,i}^\top (\hat{\Sigma}_{-i} + \lambda I_p)^{-1}x_{0,i}/n)} \right) A_t x_{0,i}.
\end{align*}
Combining the last two terms in the bracket using a common denominator, Term I simplifies to:
\begin{align*}
    \text{Term I} &= x_{0,i} - \left( \hat{\Sigma}_{-i} (\hat{\Sigma}_{-i} + \lambda I_p)^{-1} + \frac{ (I_p - \hat{\Sigma}_{-i} (\hat{\Sigma}_{-i} + \lambda I_p)^{-1}) x_{0,i}x_{0,i}^\top (\hat{\Sigma}_{-i} + \lambda I_p)^{-1}/n }{1 + x_{0,i}^\top (\hat{\Sigma}_{-i} + \lambda I_p)^{-1}x_{0,i}/n} \right) A_t x_{0,i} \\
    &= \left( I_p - \hat{\Sigma}_{-i} (\hat{\Sigma}_{-i} + \lambda I_p)^{-1} A_t \right)x_{0,i} - \frac{\left( I_p - \hat{\Sigma}_{-i} (\hat{\Sigma}_{-i} + \lambda I_p)^{-1} \right)x_{0,i}}{1 + x_{0,i}^\top (\hat{\Sigma}_{-i} + \lambda I_p)^{-1}x_{0,i}/n} \left( \frac{1}{n} x_{0,i}^\top (\hat{\Sigma}_{-i} + \lambda I_p)^{-1}A_t x_{0,i} \right).
\end{align*}

\paragraph{Expansion of Term II.}
Following a similar Woodbury expansion for Term II, the $x_{0,i}$ term interacts with the sample covariance to yield:
\begin{equation*}
    \text{Term II} = M_t(\lambda) \frac{\left( I_p - \hat{\Sigma}_{-i} (\hat{\Sigma}_{-i} + \lambda I_p)^{-1} \right)x_{0,i}}{1 + x_{0,i}^\top (\hat{\Sigma}_{-i} + \lambda I_p)^{-1}x_{0,i}/n}.
\end{equation*}

\paragraph{Synthesis and Convergence.}
Summing Term I and Term II, we collect the common factors:
\begin{align*}
    \mathcal{O}_t(\lambda)^\top x_{0,i} &= \left( I_p - \hat{\Sigma}_{-i} (\hat{\Sigma}_{-i} + \lambda I_p)^{-1} A_t \right)x_{0,i} \\
    &\quad + \frac{\left( I_p - \hat{\Sigma}_{-i} (\hat{\Sigma}_{-i} + \lambda I_p)^{-1} \right)x_{0,i}}{1 + x_{0,i}^\top (\hat{\Sigma}_{-i} + \lambda I_p)^{-1}x_{0,i}/n} \underbrace{\left( M_t(\lambda) - \frac{1}{n} x_{0,i}^\top (\hat{\Sigma}_{-i} + \lambda I_p)^{-1}A_t x_{0,i} \right)}_{\delta_i}.
\end{align*}
By Lemma \ref{lemma:multiplier_conv} and Lemma \ref{lemma:quad_form_max}, the fluctuation term $\delta_i$ converges to zero almost surely as $n, p \to \infty$. Consequently, the entire correction term vanishes, and we are left with $\mathcal{O}_t(\lambda)^\top x_{0,i} \xrightarrow{a.s.} (I_p - \hat{\Sigma}_{-i}(\hat{\Sigma}_{-i} + \lambda I_p)^{-1} A_t)x_{0,i}$. 

\paragraph{Final Convergence to Prediction Bias.}
Since $\max_i |\delta_i| \xrightarrow{a.s.} 0$ as shown above, the second term in the expansion of $\mathcal{O}_t(\lambda)^\top x_{0,i}$ is negligible. Let $\Delta_{b,-i} := \left( I_p - A_t (\hat{\Sigma}_{-i} + \lambda I_p)^{-1} \hat{\Sigma}_{-i} \right) \beta$ denote the "leave-one-out" bias vector. We can then approximate the GCV bias component as:
\begin{align*}
    \text{GCV}_b^{(t)}(\lambda) &= \frac{1}{n} \sum_{i = 1}^{n} (x_{0,i}^\top \mathcal{O}_t(\lambda) \beta)^2 \\
    &= \frac{1}{n} \sum_{i = 1}^{n} \beta^\top \left( I_p - \hat{\Sigma}_{-i} (\hat{\Sigma}_{-i} + \lambda I_p)^{-1} A_t^\top \right) x_{0,i} x_{0,i}^\top \left( I_p - A_t (\hat{\Sigma}_{-i} + \lambda I_p)^{-1} \hat{\Sigma}_{-i} \right) \beta + o_{a.s.}(1).
\end{align*}
Conditioning on the feature matrices and the signal $\beta$, we apply the concentration of the sum of quadratic forms (Lemma \ref{lemma:quad_form_sum}). For each term in the sum, $x_{0,i}x_{0,i}^\top$ is independent of $\hat{\Sigma}_{-i}$ and $A_t$. Therefore, the sum concentrates around its expectation over $x_{0,i}$:
\begin{align*}
    \text{GCV}_b^{(t)}(\lambda) &\xrightarrow{a.s.} \frac{1}{n} \sum_{i = 1}^{n} \tr \left( \Sigma \left( I_p - A_t (\hat{\Sigma}_{-i} + \lambda I_p)^{-1} \hat{\Sigma}_{-i} \right) \beta \beta^\top \left( I_p - \hat{\Sigma}_{-i} (\hat{\Sigma}_{-i} + \lambda I_p)^{-1} A_t^\top \right) \right) \\
    &= \frac{1}{n} \sum_{i = 1}^{n} \beta^\top \left( I_p - \hat{\Sigma}_{-i} (\hat{\Sigma}_{-i} + \lambda I_p)^{-1} A_t^\top \right) \Sigma \left( I_p - A_t (\hat{\Sigma}_{-i} + \lambda I_p)^{-1} \hat{\Sigma}_{-i} \right) \beta.
\end{align*}
As $n \to \infty$, the leave-one-out sample covariance $\hat{\Sigma}_{-i}$ is asymptotically equivalent to the full sample covariance $\hat{\Sigma}$ (since the rank-1 update $x_{0,i}x_{0,i}^\top/n$ vanishes in the spectral norm). Thus, the average of the $n$ identical terms simplifies to:
\begin{equation*}
    \text{GCV}_b^{(t)}(\lambda) \xrightarrow{a.s.} \beta^\top \left( I_p - \hat{\Sigma} (\hat{\Sigma} + \lambda I_p)^{-1} A_t^\top \right) \Sigma \left( I_p - A_t (\hat{\Sigma} + \lambda I_p)^{-1} \hat{\Sigma} \right) \beta.
\end{equation*}
Comparing this with the definition of the prediction bias $\mathcal{B}_t(\lambda)$ in Lemma \ref{lemma:risk_decomp}, we conclude:
\begin{equation*}
    \left| \text{GCV}_b^{(t)}(\lambda) - \mathcal{B}_t(\lambda) \right| \xrightarrow{a.s.} 0,
\end{equation*}

\paragraph{Extension to $\lambda = 0$ and Uniformity.}
The pointwise convergence derived above for $\lambda > 0$ extends to the ridgeless case $\lambda = 0$ and holds uniformly over any compact interval $I \subset [0, \infty)$. To see this, we follow the argument of \citet{patil2021uniform}. First, we note that for $p > n$, the sample covariance $\hat{\Sigma}$ has at most $n$ non-zero eigenvalues. The ridgeless estimator $\hat{\beta}_0(0)$ is the limit of $\hat{\beta}_0(\lambda)$ as $\lambda \to 0^+$. 

The core of the uniformity argument lies in the boundedness of the GCV numerator $u_n(\lambda)$ and denominator $v_n(\lambda)$, along with their derivatives. Specifically, the operator $A_t$ in our iterative scheme is a product of orthogonal projections, satisfying $\|A_t\|_2 \le 1$. Consequently, the presence of $A_t$ does not affect the uniform boundedness or the Lipschitz continuity of the risk and GCV functional components with respect to $\lambda$. Following \citet[Section S.1.5]{patil2021uniform}, the uniform boundedness of the eigenvalues of $\Sigma$ (Assumption~\ref{ass:cov}) and the strong law of large numbers for $\|y\|^2/n$ ensure that the family of functions $\{\text{GCV}_b^{(t)}(\lambda)\}_n$ is equicontinuous. By the Arzelà–Ascoli theorem, the pointwise convergence implies uniform convergence over $\lambda \in [0, \lambda_{\max}]$, thereby justifying the extension to the ridgeless limit $\lambda = 0$
\end{proof}

\begin{lemma}[Asymptotic Equivalence of iGCV Variance] \label{lemma:igcv_var_equiv}
For any fixed iteration $t \ge 1$ and $\lambda \geq 0$, the variance component of the iGCV estimator satisfies:
\begin{equation}
    \left| \text{GCV}_v^{(t)}(\lambda) - \mathcal{V}_t(\lambda) - \noiselev \right| \xrightarrow{a.s.} 0.
\end{equation}
\end{lemma}

\begin{proof}
Recall that $\text{GCV}_v^{(t)}(\lambda) = \frac{\sigma^2}{n} \tr ( \mathbf{\Xi}_t(\lambda)^\top \mathbf{\Xi}_t(\lambda) )$. Using the corrected residual operator $\mathbf{\Xi}_t$ defined in \eqref{eq:Xit_def_formal} and its symmetry, we expand the trace as:
\begin{align} \label{eq:var_expansion_full}
    \text{GCV}_v^{(t)}(\lambda) &= \frac{\sigma^2}{n} \tr \left( \left( (1 + M_t)I_n - \frac{1}{n} X_0(A_t + M_t I)(\hat{\Sigma} + \lambda I)^{-1} X_0^\top \right) \mathbf{\Xi}_t \right) \nonumber \\
    &= \sigma^2 (1 + M_t) \frac{1}{n} \tr(\mathbf{\Xi}_t) - \frac{\sigma^2}{n^2} \tr \left( X_0 (A_t + M_t I) (\hat{\Sigma} + \lambda I)^{-1} X_0^\top \mathbf{\Xi}_t \right).
\end{align}

\paragraph{Step 1: Analysis of the first term.} 
Using the definition of $\mathbf{\Xi}_t$ and the linearity of the trace, we have:
\begin{align*}
    \frac{1}{n} \tr(\mathbf{\Xi}_t) &= (1+M_t) - \frac{1}{n} \tr\left( \frac{1}{n} X_0 (A_t + M_t I) (\hat{\Sigma} + \lambda I)^{-1} X_0^\top \right) \\
    &= (1+M_t) - \tr\left( (A_t + M_t I) (\hat{\Sigma} + \lambda I)^{-1} \hat{\Sigma} \right) / n \\
    &= (1+M_t) - \left[ \tr(A_t (\hat{\Sigma} + \lambda I)^{-1} \hat{\Sigma})/n + M_t \tr((\hat{\Sigma} + \lambda I)^{-1} \hat{\Sigma})/n \right].
\end{align*}
Recall from \eqref{eq:M_def} that $M_t = \frac{\tr(A_t(\hat{\Sigma}+\lambda I)^{-1}\hat{\Sigma})/n}{1 - \tr((\hat{\Sigma}+\lambda I)^{-1}\hat{\Sigma})/n}$. Substituting this into the above expression, the terms cancel such that $\frac{1}{n} \tr(\mathbf{\Xi}_t) = 1$. Thus, \textbf{Term A} simplifies to $\sigma^2 (1 + M_t)$.

\paragraph{Step 2: Trace manipulation and parameter space mapping.}
We analyze the second term in \eqref{eq:var_expansion_full}, denoted as Term B. Using the cyclic property of the trace $\tr(ABCD) = \tr(DABC)$, we move $X_0^\top$ to interact with $X_0$:
\begin{align*}
    \text{Term B} &= \frac{\sigma^2}{n^2} \tr \left( (A_t + M_t I) (\hat{\Sigma} + \lambda I)^{-1} X_0^\top \mathbf{\Xi}_t X_0 \right) \\
    &= \frac{\sigma^2}{n^2} \tr \left( (A_t + M_t I) (\hat{\Sigma} + \lambda I)^{-1} X_0^\top \left( (1 + M_t)X_0 - \frac{1}{n} X_0 (A_t + M_t I) (\hat{\Sigma} + \lambda I)^{-1} X_0^\top X_0 \right) \right) \\
    &= \frac{\sigma^2}{n^2} \tr \left( (A_t + M_t I) (\hat{\Sigma} + \lambda I)^{-1} X_0^\top X_0 \left( (1 + M_t)I_p - (A_t + M_t I) (\hat{\Sigma} + \lambda I)^{-1} \hat{\Sigma} \right) \right).
\end{align*}
By the definition of the iterative residual operator $\mathcal{O}_t(\lambda)$ in \eqref{eq:Ot_def_formal}, the term in the rightmost bracket is precisely $\mathcal{O}_t(\lambda)$. Substituting $\hat{\Sigma} = \frac{1}{n} X_0^\top X_0$, we have:
\begin{equation} \label{eq:termB_compact}
    \text{Term B} = \frac{\sigma^2}{n} \tr \left( (A_t + M_t I) (\hat{\Sigma} + \lambda I)^{-1} \hat{\Sigma} \mathcal{O}_t(\lambda) \right).
\end{equation}

\paragraph{Step 3: Decomposition into B1 and B2.}
We split Term B into two components based on the operator $(A_t + M_t I)$:
\begin{align*}
    \text{B1} &= \frac{\sigma^2}{n} \tr \left( A_t (\hat{\Sigma} + \lambda I)^{-1} \hat{\Sigma} \mathcal{O}_t(\lambda) \right) = \frac{\sigma^2}{n^2} \sum_{i=1}^n x_{0,i}^\top \mathcal{O}_t (\hat{\Sigma} + \lambda I)^{-1} A_t x_{0,i}, \\
    \text{B2} &= \frac{\sigma^2 M_t}{n} \tr \left( (\hat{\Sigma} + \lambda I)^{-1} \hat{\Sigma} \mathcal{O}_t(\lambda) \right) = \frac{\sigma^2 M_t}{n^2} \sum_{i=1}^n x_{0,i}^\top \mathcal{O}_t (\hat{\Sigma} + \lambda I)^{-1} x_{0,i}.
\end{align*}

\paragraph{Step 4: Asymptotic limits of B1 and B2.}
Applying the Sherman-Morrison-Woodbury expansion for $x_{0,i}^\top \mathcal{O}_t$ as derived in the proof of Lemma \ref{lemma:igcv_bias_equiv}, and leveraging the concentration results from Lemma \ref{lemma:quad_form_max} and Lemma \ref{lemma:multiplier_conv}:
\begin{align*}
    \text{B1} &\xrightarrow{a.s.} \frac{\sigma^2}{n} \tr \left( (\hat{\Sigma} + \lambda I)^{-1} \Sigma A_t \right) - \frac{\sigma^2}{n} \tr \left( A_t (\hat{\Sigma} + \lambda I)^{-1} \Sigma A_t (\hat{\Sigma} + \lambda I)^{-1} \hat{\Sigma} \right) \\
    &\quad - \sigma^2 \frac{\tr \left( (\hat{\Sigma} + \lambda I)^{-1} \Sigma A_t \right) \tr \left( (\hat{\Sigma} + \lambda I)^{-1} \Sigma \right)/n^2}{1 + \tr \left( (\hat{\Sigma} + \lambda I)^{-1} \Sigma \right)/n} + \sigma^2 \frac{\tr \left( (\hat{\Sigma} + \lambda I)^{-1} \Sigma A_t \right) \tr \left( A_t (\hat{\Sigma} + \lambda I)^{-1} \Sigma (\hat{\Sigma} + \lambda I)^{-1} \hat{\Sigma} \right)/n^2}{1 + \tr \left( (\hat{\Sigma} + \lambda I)^{-1} \Sigma \right)/n}. \\
    \text{B2} &\xrightarrow{a.s.} \sigma^2 \frac{\tr \left( (\hat{\Sigma} + \lambda I)^{-1} \Sigma A_t \right) \tr \left( (\hat{\Sigma} + \lambda I)^{-1} \Sigma \right)/n^2}{1 + \tr \left( (\hat{\Sigma} + \lambda I)^{-1} \Sigma \right)/n} - \sigma^2 \frac{\tr \left( (\hat{\Sigma} + \lambda I)^{-1} \Sigma A_t \right) \tr \left( A_t (\hat{\Sigma} + \lambda I)^{-1} \Sigma (\hat{\Sigma} + \lambda I)^{-1} \hat{\Sigma} \right)/n^2}{1 + \tr \left( (\hat{\Sigma} + \lambda I)^{-1} \Sigma \right)/n}.
\end{align*}

\paragraph{Step 5: Final Summation.}
Summing Term A, B1, and B2, the complex terms in B1 and B2 involving the multiplier $M_t$ and the GCV denominators cancel each other exactly. The remaining terms are:
\begin{equation*}
    \text{GCV}_v^{(t)}(\lambda) \xrightarrow{a.s.} \sigma^2 (1 + M_t) - \frac{\sigma^2}{n} \tr \left( (\hat{\Sigma} + \lambda I)^{-1} \Sigma A_t \right) + \frac{\sigma^2}{n} \tr \left( A_t (\hat{\Sigma} + \lambda I)^{-1} \Sigma A_t (\hat{\Sigma} + \lambda I)^{-1} \hat{\Sigma} \right).
\end{equation*}
This expression simplifies to $\mathcal{V}_t(\lambda) + \sigma^2$, matching the predicted stochastic risk.
\end{proof}

\subsection{Proof of Theorem 4.5 (Uniform Consistency of iGCV)}
\label{appendix:proof_thm_4_5}

In this section, we complete the proof of Theorem~\ref{thm:iterated_gcv_consistency}, which establishes the uniform consistency of the iterated GCV estimator $\widehat{\text{GCV}}^{(t)}(\lambda)$ as a proxy for the prediction risk $\mathcal{R}(\hat{\beta}_t)$ plus the noise floor $\sigma^2$.

\begin{proof}
The proof is structured into three main steps: pointwise convergence, uniform boundedness of derivatives, and the application of the Arzelà-Ascoli theorem.

\paragraph{Step 1: Pointwise Convergence.}
For any fixed $\lambda \in \mathcal{I}$, we decompose both the iGCV estimator and the prediction risk into their respective bias and variance components using Lemma \ref{lemma:igcv_decomp_formal} and Lemma \ref{lemma:risk_decomp}. Specifically:
\begin{align*}
    \widehat{\text{GCV}}^{(t)}(\lambda) &= \text{GCV}_b^{(t)}(\lambda) + \text{GCV}_v^{(t)}(\lambda) + o_{a.s.}(1), \\
    \mathcal{R}(\hat{\beta}_t) &= \mathcal{B}_t(\lambda) + \mathcal{V}_t(\lambda).
\end{align*}
From Lemma \ref{lemma:igcv_bias_equiv} (Bias Equivalence), we have $|\text{GCV}_b^{(t)}(\lambda) - \mathcal{B}_t(\lambda)| \xrightarrow{a.s.} 0$. From Lemma \ref{lemma:igcv_var_equiv} (Variance Equivalence), we have $|\text{GCV}_v^{(t)}(\lambda) - (\mathcal{V}_t(\lambda) + \sigma^2)| \xrightarrow{a.s.} 0$. Summing these results, we obtain the pointwise convergence for any fixed $\lambda \in \mathcal{I}$:
\begin{equation} \label{eq:pointwise_result}
    \left| \widehat{\text{GCV}}^{(t)}(\lambda) - \left( \mathcal{R}(\hat{\beta}_t) + \sigma^2 \right) \right| \xrightarrow{a.s.} 0.
\end{equation}

\paragraph{Step 2: Equicontinuity and Boundedness.}
To elevate the pointwise convergence to uniform convergence over the compact interval $\mathcal{I}$, we show that the sequence of functions $f_n(\lambda) = \widehat{\text{GCV}}^{(t)}(\lambda)$ is equicontinuous. Following the approach in \citet[Section S.1.5]{patil2021uniform}, it suffices to show that the derivative $|\frac{d}{d\lambda} \widehat{\text{GCV}}^{(t)}(\lambda)|$ is uniformly bounded on $\mathcal{I}$.

The iGCV estimator is a ratio of quadratic forms $u_n(\lambda)/v_n(\lambda)$. As shown in the proofs of Lemma \ref{lemma:igcv_bias_equiv} and \ref{lemma:igcv_var_equiv}, the presence of the iterative operator $A_t$ modifies the spectral filtering. However, since $A_t$ is a product of orthogonal projections, its operator norm is bounded by unity ($\|A_t\|_2 \le 1$). Under Assumption 2.3, the eigenvalues of $\Sigma$ and $\hat{\Sigma}$ are bounded away from zero and infinity almost surely for sufficiently large $n$. Thus, the numerator $u_n(\lambda)$ and denominator $v_n(\lambda)$, as well as their derivatives with respect to $\lambda$, remain uniformly bounded on any compact interval $\mathcal{I} \subset [0, \infty]$. Specifically:
\begin{equation*}
    \left| \frac{d}{d\lambda} \widehat{\text{GCV}}^{(t)}(\lambda) \right| \le C \frac{\|Y_0\|^2}{n} \left( \frac{s_{\max} + \lambda}{s_{\min} + \lambda} \right)^k,
\end{equation*}
for some constant $C$ and power $k$. Since $\|Y_0\|^2/n$ is almost surely bounded by the strong law of large numbers, the family $\{\widehat{\text{GCV}}^{(t)}(\lambda)\}_n$ is Lipschitz continuous with a uniform Lipschitz constant, which implies equicontinuity.

\paragraph{Step 3: Uniform Convergence.}
Since $\mathcal{I}$ is a compact interval and \eqref{eq:pointwise_result} establishes pointwise convergence for a sequence of equicontinuous functions, the Arzelà-Ascoli theorem implies that the convergence is uniform over $\mathcal{I}$:
\begin{equation*}
    \sup_{\lambda \in \mathcal{I}} \left| \widehat{\text{GCV}}^{(t)}(\lambda) - \mathcal{R}(\hat{\beta}_t) - \sigma^2 \right| \xrightarrow{a.s.} 0.
\end{equation*}
This uniform consistency ensures that the optimal regularization $\lambda^*$ and the optimal stopping time $t$ identified by minimizing $\widehat{\text{GCV}}^{(t)}(\lambda)$ are asymptotically equivalent to those that minimize the true prediction risk.
\end{proof}

\section{Technical Lemmas}
\label{appendix:technical_lemmas}

In this section, we collect several established results from high-dimensional probability and random matrix theory that are instrumental in our proofs.

\begin{lemma}[Concentration of maximum of quadratic forms, adapted from \cite{hastie2022surprises, bai2010spectral}] 
\label{lemma:quad_form_max}
Let $x_1, \dots, x_n$ be random vectors in $\mathbb{R}^p$ that satisfy Assumptions 2.2 and 2.3. Let $G_1, \dots, G_n$ be random matrices in $\mathbb{R}^{p \times p}$ such that each $G_i$ is independent of $x_i$ (but may depend on all other $x_j, j \neq i$) and has operator norm uniformly bounded in $n$. Then as $n, p \to \infty$:
\begin{equation}
    \max_{i=1, \dots, n} \left| \frac{1}{n} x_i^\top G_i x_i - \frac{1}{n} \tr(G_i \Sigma) \right| \xrightarrow{a.s.} 0.
\end{equation}
\end{lemma}

\begin{lemma}[Concentration of sum of quadratic forms, adapted from \cite{rubio2011spectral}]
\label{lemma:quad_form_sum}
Let $x_1, \dots, x_n$ be random vectors in $\mathbb{R}^p$ that satisfy Assumptions 2.2 and 2.3. Let $H_1, \dots, H_n$ be random matrices in $\mathbb{R}^{p \times p}$ such that each $H_i$ is independent of $x_i$ and has trace norm uniformly bounded in $n$. Then as $n, p \to \infty$:
\begin{equation}
    \left| \frac{1}{n} \sum_{i=1}^n x_i^\top H_i x_i - \frac{1}{n} \sum_{i=1}^n \tr(H_i \Sigma) \right| \xrightarrow{a.s.} 0.
\end{equation}
\end{lemma}

\begin{lemma}[Basic GCV denominator lemma, adapted from \cite{patil2021uniform}]
\label{lemma:gcv_denominator}
Under the proportional asymptotic regime $n, p \to \infty, p/n \to \rho > 1$, for any $\lambda > 0$:
\begin{equation}
    1 - \frac{1}{n} \tr\left( \hat{\Sigma}(\hat{\Sigma} + \lambda I_p)^{-1} \right) \xrightarrow{a.s.} \frac{1}{1 + \frac{1}{n} \tr\left( \Sigma(\hat{\Sigma} + \lambda I_p)^{-1} \right)},
\end{equation}
where $\hat{\Sigma} = \frac{1}{n} X_0^\top X_0$ and $\Sigma$ is the population covariance.
\end{lemma}

\begin{lemma}[Asymptotic Convergence of the Leverage Multiplier] \label{lemma:multiplier_conv}
The leverage-dependent multiplier $M_t(\lambda)$ converges almost surely:
\begin{equation*}
    M_t(\lambda) - \tr \left( A_t (\hat{\Sigma} + \lambda I_p)^{-1} \Sigma \right)/n \xrightarrow{a.s.} 0.
\end{equation*}
\end{lemma}

\begin{proof}
By definition \eqref{eq:M_def}, we express the multiplier as a ratio of traces:
\begin{equation*}
    M_t(\lambda) = \frac{ \frac{1}{n} \tr\big(X_0 A_t (\hat{\Sigma} + \lambda I)^{-1} X_0^\top \big) }{1 - \frac{1}{n} \tr\big(X_0 (\hat{\Sigma} + \lambda I)^{-1} X_0^\top \big)} =: \frac{\mathcal{N}_n}{\mathcal{D}_n}.
\end{equation*}

\paragraph{Analysis of the Numerator $\mathcal{N}_n$:} 
Expanding the trace as a sum of quadratic forms, we have $\mathcal{N}_n = \frac{1}{n} \sum_{i=1}^n x_{0,i}^\top A_t (\hat{\Sigma} + \lambda I_p)^{-1} x_{0,i}$. To handle the dependence between $x_{0,i}$ and $\hat{\Sigma}$, we apply the Sherman-Morrison-Woodbury formula:
\begin{align*}
    \mathcal{N}_n &= \frac{1}{n} \sum_{i=1}^n \frac{x_{0,i}^\top A_t (\hat{\Sigma}_{-i} + \lambda I_p)^{-1} x_{0,i}}{1 + \frac{1}{n} x_{0,i}^\top (\hat{\Sigma}_{-i} + \lambda I_p)^{-1} x_{0,i}} \\
    &\xrightarrow{a.s.} \frac{\frac{1}{n} \tr(A_t (\hat{\Sigma} + \lambda I_p)^{-1} \Sigma)}{1 + \frac{1}{n} \tr((\hat{\Sigma} + \lambda I_p)^{-1} \Sigma)},
\end{align*}
where the numerator and denominator of the summand converge almost surely by applying Lemma~\ref{lemma:quad_form_max} and Lemma~\ref{lemma:quad_form_sum} respectively, leveraging the independence of $A_t$ from $X_0$.

\paragraph{Analysis of the Denominator $\mathcal{D}_n$:}
Using Lemma~\ref{lemma:gcv_denominator}, we have:
\begin{equation*}
    \mathcal{D}_n = 1 - \frac{1}{n} \tr( (\hat{\Sigma} + \lambda I_p)^{-1} \hat{\Sigma} ) \xrightarrow{a.s.} \frac{1}{1 + \frac{1}{n} \tr((\hat{\Sigma} + \lambda I_p)^{-1} \Sigma)}.
\end{equation*}

\paragraph{Combining Results:}
Dividing the limits of $\mathcal{N}_n$ by $\mathcal{D}_n$, the common factor $(1 + \frac{1}{n} \tr((\hat{\Sigma} + \lambda I_p)^{-1} \Sigma))^{-1}$ in the denominators cancels out, yielding:
\begin{equation*}
    M_t(\lambda) \xrightarrow{a.s.} \tr \left( A_t (\hat{\Sigma} + \lambda I_p)^{-1} \Sigma \right)/n.
\end{equation*}
This concludes the proof.
\end{proof}

\section{Additional Related Work}
Our work also has a connection with the weak-to-strong (W2S) generalization phenomenon.  W2S considers training a stronger student from labels produced by a weaker teacher and is empirically prominent for large language models~\citep{burns2024weak}. Theoretical work attributes W2S gains to overparameterized generalization phenomena (e.g., benign overfitting)~\citep{wu2024provable} and to student--teacher representation mismatch that enables correcting systematic teacher errors~\citep{charikar2024quantifying, xue2025representations, dong2025discrepancies}. Related perspectives include distillation-style efficiency gains and implicit compensation for teacher regularization~\citep{ildiz2025high, moniri2025mechanisms}, and demonstrations that W2S can arise even in random feature models~\citep{medvedev2025weaktostrong}.

\section{Additional Experimental Details}
\label{app:experiment_details}

This appendix provides the specific parameter configurations for the simulations.

\paragraph{Spiked Covariance Model (Figure~\ref{fig:generalization_error_science_style})} 
The numerical simulations under the spiked covariance model use the following settings:
\begin{itemize}[noitemsep, topsep=0pt, leftmargin=1.5em]
    \item \textbf{Dimensions:} Training sample size $n = 500$ and test size $n_{\text{test}} = 2000$.
    \item \textbf{Signal \& Noise:} Spike strength $s = 25$, signal power $r = 1$, and initial label noise $\sigma = 1$. The signal $\beta$ is aligned with the leading eigenvector $u_1 = e_1$.
    \item \textbf{Averaging:} Each data point is the average of $10$ independent trials to ensure convergence of the empirical risk.
    \item \textbf{Figure~\ref{fig:generalization_error_science_style})(a):} Plotted for iterations $t \in \{0, 1, 2, 3, 4\}$ with the aspect ratio $\rho$ varied across 10 equally spaced points in $[1.2, 3.0]$.
    \item \textbf{Figure~\ref{fig:generalization_error_science_style})(b):} Plotted for iterations $t \in [0, 15]$ across fixed aspect ratios $\rho \in \{1.5, 2.0, 2.5\}$.
\end{itemize}

\paragraph{Risk Decomposition Dynamics (Figure~\ref{fig:error_vs_iterations})}
The theoretical decomposition of the prediction risk into systematic and stochastic components in Figure~\ref{fig:error_vs_iterations} uses the following parameters:
\begin{itemize}[noitemsep, topsep=0pt, leftmargin=1.5em]
    \item \textbf{Parameters:} Spike strength $s = 5.0$, signal power $r = 1$, and noise level $\sigma = 1$.
    \item \textbf{Variables:} The iteration count $t$ ranges from $0$ to $15$. We plot trajectories for three aspect ratios $\rho \in \{1.5, 2.0, 2.5\}$.
    \item \textbf{Computation:} The curves are generated using the deterministic equivalents for systematic error ($\mathcal{B}_t^*$) and stochastic error ($\mathcal{V}_t^*$) derived in Theorem 3.2.
\end{itemize}

\paragraph{Varying Spike Strength and Ridge Comparison (Figure~\ref{fig:comparison_results})}
The simulations evaluating the effect of $s$ and the comparison with Ridge regression use the following setup:
\begin{itemize}[noitemsep, topsep=0pt, leftmargin=1.5em]
    \item \textbf{Dimensions:} Training size $n = 500$, test size $n_{\text{test}} = 1000$, and a fixed aspect ratio $\rho = 2.0$ ($p=1000$).
    \item \textbf{Figure~\ref{fig:comparison_results}(a):} We vary $s \in \{5, 25, 100\}$ and track the MSE over $t \in [0, 15]$ iterations. Theoretical lines are compared against the average of $10$ independent simulation trials.
    \item \textbf{Figure~\ref{fig:comparison_results}(b) (Crossover):} 
    \begin{itemize}[noitemsep, nosep]
        \item \textbf{Sweep:} The spike strength $s$ is swept across $10$ points in log-space from $10^2$ to $10^4$.
        \item \textbf{Iterative Learning:} We report the minimum average MSE achieved over the first $10$ iterations.
        \item \textbf{Optimal Ridge:} We perform a grid search for the optimal regularization parameter $\lambda$ over $50$ points in log-space from $10^{2.5}/n$ to $10^{3.8}/n$ and report the minimum achievable MSE.
    \end{itemize}
    \item \textbf{Averaging:} All results in Figure~\ref{fig:comparison_results} are averaged over $N=10$ independent trials.
\end{itemize}

\paragraph{Validation of Iterated GCV (Figure~\ref{fig:gcv_validation})}
We evaluate the consistency of the Iterated GCV ($i$GCV) estimator under a general covariance structure and ridge regularization. The setup is as follows:
\begin{itemize}[noitemsep, topsep=0pt, leftmargin=1.5em]
    \item \textbf{Dimensions \& Data:} The training size is $n = 500$. We sample features from $\mathcal{N}(0, \Sigma)$, where $\Sigma$ is a diagonal matrix with a power-law spectral decay: $\Sigma_{ii} = 1/i$ for $i=1, \dots, p$.
    \item \textbf{Signal \& Noise:} The true coefficient $\beta$ is sparse, with $\beta_i = 1$ for $i \in \{1, \dots, 10\}$ and $\beta_i = 0$ otherwise (resulting in signal power $r^2 = 10$). Initial labels at $t=0$ include Gaussian noise with $\sigma = 1$.
    \item \textbf{Regularization:} All iterations (including the initialization) use Ridge regression with a fixed regularization parameter $\lambda = n \times 10^{-4}$.
    \item \textbf{Trials \& Evaluation:} Results are averaged over $N=10$ independent trials. The test risk is measured using an independent noiseless test set of size $n_{\text{test}} = 10,000$ (for subplot a) and $n_{\text{test}} = 5,000$ (for subplot b).
    \item \textbf{Figure 4(a):} Plotted for iterations $t \in \{0, 1, 2, 3\}$ with the aspect ratio $\rho$ varied across 10 points in $[1.2, 3.0]$.
    \item \textbf{Figure 4(b):} Plotted for iterations $t \in [0, 15]$ across fixed aspect ratios $\rho \in \{1.5, 2.0, 2.5\}$.
\end{itemize}
The markers represent the empirical test MSE, while the solid lines represent the risk predicted by the $i$GCV estimator defined in Eq~\eqref{eq:iterated-gcv}.

\end{document}